%% file: Arxiv.tex
\documentclass[11pt,a4paper]{article}

\usepackage{cite}

\usepackage{amsmath} 
\usepackage{amssymb}  
\usepackage{amsthm} 
\usepackage{url} 
\newcommand{\Pc}{\mathcal{P}}
\newcommand{\Qc}{\mathcal{Q}}
\newcommand{\Lc}{\mathcal{L}}
\newcommand{\V}{\mathcal{V}}

\newcommand{\R}{\mathbb{R}}
\newcommand{\A}{\mathcal{A}}
\newcommand{\B}{\mathcal{B}}

\newcommand{\Se}{\mathbb{S}}
\newtheorem{problem}{Problem}
\newtheorem{proposition}{Proposition}
\newtheorem{theorem}{Theorem}
\newtheorem{definition}{Definition}
\newtheorem{remark}{Remark}
\usepackage{layouts}
\usepackage{algorithm}
\usepackage{algpseudocode}
\newcommand{\first}[1]{\textrm{first}\,#1}
\newcommand{\rest}[1]{\textrm{rest}\,#1}
\newcommand{\add}[1]{\textrm{add}\,#1}
\usepackage{graphicx}
\usepackage{xcolor}
\input Acorn.fd

\def\oversl#1{
  \setbox0=\hbox{$#1$}
  \slantmathcorr=\wd0
  \hskip 0.2\slantmathcorr \overline{\hbox to 0.8\wd0{%
  \vphantom{\hbox{$#1$}}}}
  \hskip-\wd0\hbox{$#1$}
}

\def\undersl#1{
  \setbox0=\hbox{$#1$}
  \slantmathcorr=\wd0
  \underline{\hbox to 0.8\wd0{%
  \vphantom{\hbox{$#1$}}}}
  \hskip-0.8\wd0\hbox{$#1$}
}

\begin{document}

\title{A Second-Order Lower Bound for Globally Optimal 2D Registration}

\providecommand{\keywords}[1]{\textbf{\textit{Index terms---}} #1}
\author{Luca Consolini$^1$, Mattia Laurini$^1$, Marco Locatelli$^1$, Dario Lodi Rizzini$^1$}

\date{\small $^1$ Dipartimento di Ingegneria e Architettura, Universit\`a degli Studi di Parma,\\ Parco Area delle Scienze 181/A, 43124 Parma, Italy.\\ luca.consolini@unipr.it, mattia.laurini@unipr.it, marco.locatelli@unipr.it, dario.lodirizzini@unipr.it}

\maketitle

\begin{abstract}
The problem of planar registration consists in finding the
transformation that better aligns two point sets.
In our setting, the search domain is the set of planar rigid transformations
and the objective function is the sum of the distances between each
point of the transformed source set and the destination set.
We consider a Branch and Bound (BnB) method for
finding the globally optimal solution.
The algorithm recursively
splits the search domain into boxes and computes an upper and a lower
bound for the minimum value of the restricted problem.
The main contribution of this work is the introduction of a novel lower bound, the \emph{relaxation bound}, which corresponds to the
solution of a
concave relaxation of the objective function based on the
linearization of the distance. 
In the BnB we also employ the so called \emph{cheap bound}, equal to to the sum of the minimum distances
between each point of source point set, transformed according to current box, and all the candidate points in
the destination point set.
We prove, both theoretically and practically, that the novel relaxation bound dominates the cheap bound over small boxes. More precisely, from the theoretical point of view, we prove that 
the relaxation bound is a second-order approximation of the minimum value, i.e., its distance from the minimum value decreases 
quadratically with respect to the diameter of the box (see Theorem \ref{prop:relaxation_error}), while the cheap bound is a first-order one (see Proposition \ref{prop:cheap_error}). From the practical point of view, we show through different computational experiments that the addition of the relaxation bound considerably enhances the performance of the BnB algorithm, compensating the higher cost of its computation with respect to the cheap bound with a strong
reduction of the number of BnB nodes to be explored.
We stress that the newly proposed relaxation bound has been tested within a specific BnB approach but could be as well employed in other BnB approaches for this problem.
Finally, we remark that a queue-based algorithm is employed for bound computations, which allows for a considerable speed
up.
\end{abstract}


\maketitle


\section{Introduction}
\label{sec:introduction}

Point set registration is the problem of estimating the rigid transformation that better aligns two or more point sets.
Registration is a primitive for a wide range of applications, including localization and mapping~\cite{NuchterLingemannHertzbergSurmann07, MurArtalMontielTardos15,SerafinGrisetti17}, object reconstruction~\cite{HuberHerbert03, AldomaMartonTombariEtAl12} or shape recognition~\cite{GroganDahyot18}. 
In general, registration enables matching different views of the same object or scene, observed from different viewpoints. 
The research community has proposed different formulations of the problem as well as several registration algorithms. 
In the standard formulation, there are two input point sets, one called source point set and the other called destination, target or reference point set. 
The goal is to find a rigid transform that minimizes the sum of square distances between each transformed source point to its closest destination point. 
Other formulations include, for example, different domains (2D or 3D Euclidean space), input data (polylines, gaussian distributions, etc. instead of points), objective functions, or outlier rejection techniques. 
Most of registration algorithms compute solutions corresponding to local minima of the objective function and are often dependent from the initial estimation of the transformation, which is iteratively refined. 
Recently, global optimal solutions have been proposed according to the Branch and Bound (BnB) paradigm. 

In this paper, we consider a BnB algorithm for finding the globally optimal solution of planar point set registration. 
We consider as input two planar point sets, a source
and a destination set. 
The search domain is the set of planar
rototranslations
and the objective function is the sum of the distances between each
point of the transformed source set and the destination set. For
robustness, the points of the source set with the largest error are
omitted from the sum.
The proposed algorithm recursively splits the search domain into boxes and, when needed, recursively evaluates the lower and upper bounds of the minimum of the objective function for the problem restricted to each box. 
Two different lower bounds are introduced.
The \emph{cheap bound} consists of the sum of the minimum distances
between each point of source point set, after all possible
transformations represented by current box, and all the candidate points in
the destination point set. 
The cheap bound is based on the same idea of the lower bound proposed in~\cite{YangLiCampbellJia16} for the 3D registration problem. 
The \emph{relaxation bound} represents the main contribution of this work. It corresponds to the solution of a
concave relaxation of the objective function based on the
linearization of the distance. In large boxes, the cheap bound is a
better approximation of the function minimum, while, in small boxes,
the relaxation bound is much more accurate. More precisely, it will be theoretically proved that the cheap bound decreases linearly with the diameter of the box, while the relaxation bound decreases quadratically. 
A queue-based algorithm is employed to considerably speed
up the computation of the lower bounds by removing those associations between the source and the destination sets that
cannot correspond to the optimal solution of the problem for any sub-box of the current box.
Outliers are handled by trimming the items with larger residual.

\begin{figure}[t]
    \centering
    \includegraphics[width=0.9\columnwidth]{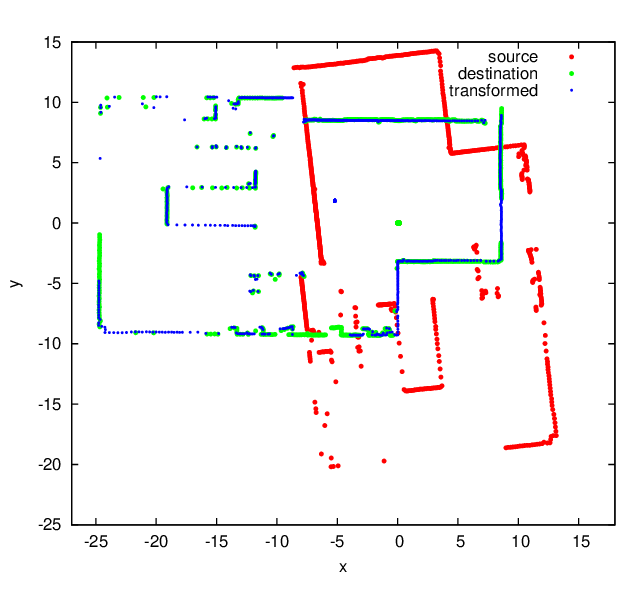}
    \caption{Example of alignment between a source $\mathcal{P}$ (red) and a destination point set $\mathcal{Q}$ (green) with the transformed point set $\mathcal{P}$ (blue).
    }
    \label{fig:scan_aligned_e80_relaxation_on}
\end{figure}

\subsection{Related work}
\label{sec:related}

The literature on registration is extensive and includes different formulations for several (sometimes loosely) related problems. 
The application contexts (computer vision, navigation, etc.) and the formulation of registration (objective function, feature-based or correlation, etc.) result into a variety of works. 
While there are several classification criteria, the categorization into local and global methods is suitable for thoroughly discussing our contribution w.r.t. the literature. 

\subsubsection{Local methods}

Iterative Closest Point (ICP)~\cite{BeslMckay92} is likely the most popular registration algorithm.
The estimation is achieved by iteratively alternating the point association and the estimation of the rigid transform that better aligns the associated points. 
As discussed in section~\ref{sec:introduction}, ICP is sensitive to initial estimation and is prone to local minima. 
Over the years, several variants, like ICP with soft assignment~\cite{GoldRangarajanLuPappuMjolsness98}, ICP with surface matching~\cite{Liu04}, affine ICP~\cite{DuZhengYingLiu10}, have been proposed. 
Generalized Procrustes Analysis (GPA) has also been proposed to simultaneously register multiple point sets in a single optimization step~\cite{ToldoBeinatCrosilla10, AleottiLodiRizziniMonicaCaselli14}. 
In some cases, the association among multiple point sets increases the robustness of estimation. 

Alternative representations to point sets have been proposed to avoid explicit assessment of correspondences. 
Biber and Strasser~\cite{BiberStrasser03} propose the Normal Distributions Transform (NDT) to model the probability of measuring a point as a mixture of normal distributions. 
Instead of establishing hard associations, NDT estimates the transformation by maximizing the probability density function of the point set matched with the mixture distribution. 
The approach has been extended to 3D point clouds~\cite{MagnussonLilienthalDuckett07} and, with ICP, is part of standard registration techniques~\cite{HolzIchimTombariRusuBehnke15}. 
Other registration techniques are based on GMMs computed on point sets~\cite{YangOngFoong15, FanYangAiXiaZhaoGaoWang16, GroganDahyot18}.

Several registration algorithms exploit rigidity of isometric transformation for selecting robust and consistent associations.
The general procedure operates in two steps. 
The first step establishes an initial set of putative associations based on geometric criteria (e.g., correspondence to closest neighbor) or similarity of descriptors. 
The second step filters the outlier associations based on rigidity constraints. 
RANSAC~\cite{FischlerBolles81} and its many variants like MLESAC~\cite{TorrZisserman00} implement this principle through a heuristic random consensus criterion.  
Coherence point drift (CPD) algorithm~\cite{MyronenkoSong10} represents point sets with a GMM and discriminates outliers by forcing GMM centroids to move coherently as a group.
Ma et al.~\cite{MaZhaoTianYuilleTu14} proposed more formally consistent assessment of associations using Vector Field Consensus (VFC). 
This method solves correspondences by interpolating a vector field between the two point sets. 
Consensus approach has also been adopted to the non-rigid registration of shapes represented as GMMs~\cite{MaQiuZhaoMaYuilleTu15,MaZhaoYuille16}. 
The hypothesize-and-verify approach is often successful in the estimation of associations, but it depends on the initial evaluation of putative correspondences. 
Moreover, it does not provide any guarantee of optimality of the solution. 

\subsubsection{Global methods} 

Global registration methods search the rigid transformation between two point sets on the complete domain of solutions. 
Heuristic global registration algorithms are based on particle swarm optimization~\cite{KhanNystrom10}, particle filtering~\cite{SandhuDambrevilleTannenbaum10}, simulated annealing~\cite{LuckLittleHoff00}. 

Another category includes the global registration methods that compare features, descriptors and orientation histograms extracted from the original point clouds. 
Spin Images~\cite{JohnsonHebert99}, FPFH~\cite{RusuBlodowBeets09, HolzIchimTombariRusuBehnke15}, SHOT~\cite{SaltiTombariDiStefano14} and shape context~\cite{BelongieMalikPuzicha02} are descriptors that could be matched according to similarity measure and used for coarse alignment of point clouds. 
Similar histogram-based methods are applied to rotation estimation of 3D polygon meshes through spherical harmonics representation~\cite{BrechbuhlerGerigKubler95, MakadiaPattersonDaniilidis06, MakadiaDaniilidis06, Kazhdan07, AlthloothiMahoorVoyles13}.
All these techniques extend the searching domain and attenuate the problem of local minima, but their computation is prone to failure or achieves extremely coarse estimation.
Moreover, the global optimality of the assessed solution is not guaranteed. 

In planar registration problem, some effective global methods exploiting specific descriptors of orientation have been proposed. 
Hough spectrum registration~\cite{CensiIocchiGrisetti05} assesses orientation through correlation of histogram measuring point collinearity. 
The extension of this method to 3D space~\cite{CensiCarpin09} suffers from observability issues due to symmetry in rotation group. 
Reyes-Lozano et al.~\cite{ReyesMedioniBayro07} propose to estimate rigid motion using geometric algebra representation of poses and tensor voting. 
Angular Radon Spectrum~\cite{LodiRizzini18} estimates the optimal rotation angle that maximizes correlation of collinearity descriptors by performing one-dimensional BnB optimization.

BnB paradigm is the basis for most of global registration methods. 
Breuel~\cite{Breuel03} proposes a BnB registration algorithm for several planar shapes in image domain. 
The shapes are handled by a \emph{matchlist} containing the shapes to be matched. 
The bounds are computed using generic geometric properties, which partially exploit pixel discretization. 
No accurate management of lower bounds is presented.
The BnB algorithm in~\cite{OlssonEnqvistKahl08, OlssonKahlOskarsson09} computes the rigid transformation that matches two point sets, under the hypothesis that the set of correspondences is given, although with outliers.
The lower bound is derived from the convex relaxation of quaternion components. 
The point-based formulation is extended also to the registration of points, lines and planes (in principle, any convex model) through similarity transformation. 
The main limitation lies in the unrealistic assumption that the correspondences between points and convex shapes are known in advance.
Brown et al.~\cite{BrownWindridgeGuillemaut19} propose the first BnB algorithm for 2D-3D registration, which is the problem of aligning the camera viewpoint given an image to a 3D model of the observed scene. 
It adopts solutions similar to~\cite{YangLiCampbellJia16} like a geometric bound and the nested BnB of translation and rotation.

To our knowledge, Go-ICP~\cite{YangLiCampbellJia16} is the most general algorithm for the estimation of the globally optimal registration.
The authors address the formulation for the 3D space domain. 
The main algorithmic contribution of~\cite{YangLiCampbellJia16} is the adoption of the uncertainty radius for measuring the distance between a reference point and a matching point transformed according to a 3D isometry belonging to a box. 
The optimal isometry is estimated by a BnB algorithm, whose lower bound consists of the sum of the minimum distances for each reference point. 
Note that, to achieve efficiency, Go-ICP algorithm adopts the Euclidean Distance Transformation, which approximates the closest-point distances on a uniform grid and limits the accuracy of optimum value estimation.

\subsection{Statement of contribution}

The main contribution of this work is the definition of the  \emph{relaxation bound}, that significantly improves the estimation of the lower bound on small sized boxes with respect to the simpler, geometrically based, \emph{cheap bound}. The superiority of the former bound over the latter is proved theoretically by showing that the relaxation bound is a second-order one (i.e., its error decreases quadratically with the diameter of the box, see Theorem \ref{prop:relaxation_error}), while the cheap bound is first-order (i.e., its error decreases linearly with the diameter of the box, see Proposition \ref{prop:cheap_error}).
 Moreover, from the practical point of view, we developed a BnB algorithm that uses the cheap bound for larger boxes and the maximum between the cheap bound and the relaxation bound for smaller ones
(and employs a rather efficient queue-based algorithm for their computation). 
Computational experiments with this BnB algorithm on simulated and real datasets confirm that the addition of the newly proposed relaxation bound strongly enhances the performance of the BnB algorithm. 
We also stress that the newly proposed relaxation bound has been tested within a specific BnB approach but could be employed in other BnB approaches for the same problem. 

Note that in this work we address the planar formulation of the registration problem, since
this is a relevant problem 
with many immediate applications: image registration~\cite{6313439}, service and telepresence robots~\cite{TriebelEtAl2016, Orlandini16}, industrial AGV localization and navigation~\cite{8103931, 8454892}.
A natural extension is the one to the 3D formulation of the registration problem. 
The main ideas of this work and, in particular, the relaxation bound, can indeed be extended to the 3D case. However, the extension is not
straightforward and, moreover, the increased complexity of the 3D problem, with the additional degrees of freedom in the rototranslations, makes convergence to the globally optimal solution harder (see also the comments in Section \ref{sec:3Dcase}). Computational tricks like the Euclidean Distance Transformation, adopted in the  Go-ICP algorithm, which speeds up computation at the cost of a more limited accuracy in the optimum value estimation, may also apply within the framework of the current work. However, a careful treatment of the 3D case requires many additional considerations which need to be addressed in a separate future work. 


\subsection{Notation}

Given a set $A \subset \R$ with finite cardinality $|A|$, we denote its elements in ascending order as $A_{(1)}, A_{(2)},\ldots, A_{(|A|)}$.
Given $x,y \in \R$, set $x \wedge y= \min\{x, y\}$.
For vectors $x,y \in \R^n$, relation $x \leq y$ is intended componentwise.
A box $B \subset \R^n$ is a set of form $B = \{x \in \R^n\ |\ x^- \leq x \leq x^+\}$ for given $x^-,x^+$, with $x^-\leq x^+$ ( if $x^- \not\leq x^+$, then $B$ is empty); moreover, the diameter of $B$ is given by $\sigma(B) = \left\| x^+ - x^- \right\|_2$, where $\|\cdot\|_2$ denotes the Euclidean norm.

\subsection{Paper organization}
The rest of the paper is organized as follows. 
Sections~\ref{sec:problem_formulation} and \ref{sec:branch_and_bound} illustrate the problem formulation and the outline of the BnB algorithm, respectively.
Section~\ref{sec:lower_bound} defines the two lower bounds used within the BnB algorithm. 
Section~\ref{sec:algorithm} describes the complete algorithm.
Section~\ref{sec_proof} presents some technical proofs.
Section~\ref{sec:results} presents the computational results on simulated point sets and on datasets acquired by laser scanner sensors. 
Finally, Section~\ref{sec:conclusion} gives the concluding remarks.

 
\section{Problem formulation}
\label{sec:problem_formulation}
%
Consider two sets $\Pc=\{P_1,P_2,\ldots,P_n\}$,
$\Qc=\{Q_1,Q_2,\ldots,Q_m\}$, with $\Pc,\Qc \subset \R^2$,
representing the coordinates of points acquired with two different measures.
Our aim is to find a planar rigid transformation so that the transformed
points $\Pc$ are as close as possible
to set $\Qc$. Given $x \in \R^2$, a generic transformation corresponding to a counter-clockwise rotation by $\theta$ followed by a translation of $z$ can be represented by function 
$T_x (z,\theta): \R^2 \times \Se^1 \to \R^2$, defined as
\[
T_x (z,\theta) = R(\theta) x + z\,,
\]
where $R(\theta)$ is the matrix corresponding to a counter-clockwise
rotation by $\theta$, i.e., 
$$R(\theta)=\begin{bmatrix} 
\cos \theta & -\sin \theta \\
\sin \theta & \cos \theta 
\end{bmatrix}.$$
Define $f_{x,Q} (z,\theta): \R^2 \times \Se^1 \to \R$ as
\[
f_{x,Q} (z,\theta)= \|T_x(z,\theta) - Q\|_2^2\,,\]
which represents the squared distance between the rototranslation $T_x(z,\theta)$ of $x$ and
point $Q\in \R^2$, and
$f_x (z,\theta): \R^2 \times \Se^1 \to \R$ as
\[
f_x (z,\theta)= \min_{Q \in \Qc}\ f_{x,Q}(z, \theta)\,,\]
which represents the squared distance between $T_x(z,\theta)$ and set
$\Qc$, that is the minimum of the distances between $T_x(z,\theta)$
and all elements of set $\Qc$.\newline
Now, let $p \in \mathbb{N}$, with $1 \leq p \leq n$, and consider the
following problem.

\begin{problem}\label{prob:glores}
\begin{subequations}
\begin{gather}
\min_{z \in \R^2,\ \theta \in \Se^1,\ S \subset \{1,\ldots,n\}}\ \sum_{i \in S} f_{P_i} (z,\theta) \label{eq:obj_fun}\\
\text{subject to} \quad |S| = p\,. \label{eq:card_cons}
\end{gather}
\end{subequations}
\end{problem}
Problem~\ref{prob:glores} consists in finding the rigid transformation
that minimizes the sum of a subset of the squared distances between
the transformed points $\Pc$ and set $\Qc$. Note that only the smallest $p$
distances are considered in the sum in~(\ref{eq:obj_fun}). 
In this way, possible outliers (that is, points of $\Pc$ that do not
correspond to any point in $\Qc$ or are obtained from erroneous measures)
are excluded from the sum defined in~\eqref{eq:obj_fun}.
The estimator defined in Problem~\ref{prob:glores} has a breakdown
point of $n - p$, that is, the estimation of the rigid transformation is
not compromised if sets $\Pc$ or $\Qc$ contain up to $n - p$ incorrect observations.  This is the same principle that is used in Least Trimmed Squares robust regression to reduce the influence of outliers. 
Problem~\ref{prob:glores} is non-convex and NP-HARD~\cite{MountNetanyahuPiatkoSilvermanWu14}. 
In the following, we will present a BnB method for finding its exact solution.
\section{Outline of BnB method}
\label{sec:branch_and_bound}
Let $B_0=\{(z,\theta) \in \R^2 \times [0,2\pi]\ |\ z^- \leq z \leq z^+\}$.  In this section, we present a BnB algorithm for solving
the following restriction of Problem~\ref{prob:glores} to the initial box $B_0$.
\begin{problem}\label{prob:glores_box}
\begin{subequations}
\begin{gather}
\min_{(z,\theta) \in B_0,\ S \subset \{1,\ldots,n\}}\ \sum_{i \in S} f_{P_i} (z,\theta) \label{eq:obj_fun_box}\\
\text{subject to} \quad |S| = p\,.
\label{eq:card_cons_box}
\end{gather}
\end{subequations}
\end{problem}
Let $\mathcal{B}$ denote the set of boxes included in $B_0 \times
[0,2\pi]$ 
and
let 
$f^*:\mathcal{B} \to \R$ be defined as follows:
\begin{equation*}
\label{eq:fstar}
f^*(B)=\min_{(z,\theta) \in B,\ S\subset \{1,\ldots,n\}\ \mid\ |S| = p}\quad \sum_{i \in S} f_{P_i} (z,\theta).
\end{equation*}
Assume that there exists a function $\phi_L:\mathcal{B} \to \R$, such
that, $(\forall B \in \mathcal{B})$
\begin{equation}
  \label{eqn_def_lb}
 \phi_L(B) \leq f^*(B)\,.
\end{equation}
We will call any $\phi_L$ satisfying~\eqref{eqn_def_lb} a
\emph{lower bound} function. 
Further, let function $r:\mathcal{B} \to \R^2 \times [0,2 \pi]$ be such that $(\forall B \in \mathcal{B})\ r(B) \in B$.  
Function $r$ returns a point within box $B$ (in our
numerical experiments we always return the center of the box).
The optimal solution of Problem~\ref{prob:glores_box} can be found
with the standard BnB Algorithm~\ref{alg_bnb} adapted
from~\cite[p.~18]{BrAndBound}. The algorithm uses a binary
tree whose nodes are associated to a restriction of
Problem~\ref{prob:glores_box} to a box, obtained by
recursively splitting the initial box $B_0$. Input parameter $\epsilon$
represents the maximum relative allowed error on the objective function for
the optimal solution and the output variable $x^*$ is an
approximation of the optimal solution with relative tolerance $\epsilon$.
In Algorithm~\ref{alg_bnb}, function $\delta: \mathcal{B} \to \R$ is
used to define the exploration policy for set $\zeta$. For instance,
in a best first search strategy, the node with the lowest lower bound
is the next to be processed, so that $\delta(\eta)=-\phi_L(\eta)$
(this is also the choice that we made throughout the paper).
\begin{algorithm}
\caption{Main BnB algorithm}
\label{alg_bnb}
\begin{algorithmic}
\State Input: $\Pc$, $\Qc$, $B_0$: data for Problem~\ref{prob:glores_box}
\State $\epsilon$: solution tolerance
\State Output: $x^*$: optimal solution 
\begin{enumerate}
\State Let $\zeta$ be a list of boxes and initialize $\zeta=\{B_0\}$.
\State Set $UB=f(r(B_0))$, and $x^*=r(B_0)$.
\State If $\zeta=\varnothing$, stop. Else set $\delta_{\max} =
  \max \{ \delta(\eta)\ |\ \eta \in \zeta \}$.
\State Select a box $\eta \in \zeta$, with $\delta(\eta)=\delta_{\max}$
  and split it into two equal smaller sub-boxes $\eta_{1}$, $\eta_{2}$ along
  the dimension of maximum length.
\State Delete $\eta$ from $\zeta$ and add $\eta_{1}$ and $\eta_{2}$ to $\zeta$.
\State Update $UB=\min\{UB, f(r(\eta_{1})), f(r(\eta_{2}))\}$.
If $UB= f(r(\eta_j))$ with $j \in \{1,2\}$, set
$x^*= r(\eta_{j})$.
\State For all $\kappa \in \zeta$, if $\phi_L(\kappa) (1+\epsilon) \geq UB$ set $\zeta=\zeta \setminus \kappa$.
\State Return to Step 3.
\end{enumerate}
\end{algorithmic}
\end{algorithm}
Note that the choice of the lower bound function
$\phi_L$ is critical to
efficiency of Algorithm~\ref{alg_bnb}.
The following property on $\phi_L$ guarantees that
Algorithm~\ref{alg_bnb} converges to a solution of
Problem~\ref{prob:glores}, with relative tolerance $\epsilon$.
\begin{equation}
\label{eqn_cond_on_lb}
\lim_{\sigma(B) \to 0}  \left(\phi_L(B)-f^*(B) \right)=0\,.
\end{equation}


\section{Definition of the lower bound function $\phi_L$}
\label{sec:lower_bound}
In the following, we will present two different choices for the lower bound function
$\phi_L$, denoted by $\phi_C$ and $\phi_R$. We call the first one the
$\emph{cheap}$ bound, since its computation time is very small and we
call the second one the $\emph{relaxation}$ bound, since it is based
on a concave relaxation of Problem~\ref{prob:glores_box}. We will show
that the first
bound is well suited for larger boxes, while the second one requires larger computing times
but is much more accurate for smaller boxes.

\subsection{Cheap bound}
Define functions $d_{\min}, d_{\max}: \B \times \R^2  \times \R^2 \to \R$ as
\begin{align*}
d_{\min}(B,P,Q)&=\min_{(z,\theta) \in B}   f_{P,Q} (z,\theta) \\
d_{\max}(B,P,Q)&=\max_{(z,\theta) \in B}   f_{P,Q} (z,\theta).
\end{align*}
If $B=\{(z,\theta)\ |\ z^- \leq z \leq z^+, \theta^- \leq \theta \leq
\theta^+\}$, then $d_{\min}(B,P,Q)$ is the minimum distance
between circle arc $\{R(\theta) P\ |\ \theta \in
[\theta^-,\theta^+]\}$ and rectangle $\{x \in \R^2\ |\ Q-z^+ \leq x
\leq Q-z^-\}$, while $d_{\max}(B,P,Q)$ is the maximum distance
between the same sets. Details about the computation of these distances will be given in Section \ref{sec:dmindmax}.
Functions $d_{\min}$ and $d_{\max}$ can be efficiently computed with the method
presented in Section~\ref{sec_proof}.
Set $d_{\min}(B,P)=\min_{Q \in \Qc} d_{\min}(B,P,Q)$, that is
the minimum distance of $P$ with respect to set $\Qc$, for $(z,\theta) \in B$.
A simple relaxation of Problem~\ref{prob:glores_box} is obtained by
choosing different parameters $(z,\theta)$ for each point of $\Pc$,
which leads to the following proposition (whose proof is not reported
here) and which defines a lower bound closely related to the one presented in~\cite{YangLiCampbellJia16}).
\begin{proposition}
Set 
\begin{equation}
\label{eqn_cheap_b}
\phi_C(B)=\sum_{i=1}^{p} \{d_{\min}(B,P)\ |\ P \in \Pc\}_{(i)}\,,
\end{equation}

then $\phi_C (B)$ is a lower bound for the solution of Problem~\ref{prob:glores_box}.
\end{proposition}

We remind that  notation $\{d_{\min}(B,P)\ |\ P \in \Pc\}_{(i)}$ denotes
the $i$-th value of set $\{d_{\min}(B,P)\ |\ P \in \Pc\}$, orderd in
ascending order with respect to its elements, so
that~\eqref{eqn_cheap_b} corresponds to the sum of the $p$ smallest
elements of this set.
As the following proposition shows, a brute-force computation of function $\phi_C(B)$ can be
done in time proportional to the cardinality of set $\Pc
\times \Qc$.
\begin{proposition}
$\phi_C(B)$ can be computed in time $O(n m)$.
\end{proposition}
\begin{proof}
Quantity $\phi_C(B)$ can be obtained by computing all distances
$d_{\min}(B,P,Q)$ for all $P \in \Pc$, which has cardinality $n$, and all
$Q \in \Qc$, which has cardinality $m$. As we will see in Section \ref{sec:dmindmax}, each $d_{\min}$ and $d_{\max}$ computation requires $O(1)$ computing time. 
\end{proof}
In Section~\ref{sec:algorithm}, we will present an algorithm that
allows to reduce considerably the number of computations of
distances $d_{\min}(B,P,Q)$ required for computing $\phi_C$. 
The following proposition, whose proof is reported in Section~\ref{sec_proof},
shows that $\phi_C$ satisfies~\eqref{eqn_cond_on_lb} and that the
error $f^*(B) - \phi_C(B)$ is bounded by a term which is linear with respect to the diameter of
$B$.
\begin{proposition}\label{prop:cheap_error}
There exists a constant $\Gamma_C > 0$, dependent on problem data, such that, for any box $B \in \mathcal{B}$,
\[
|f^*(B) - \phi_C(B)| \leq \Gamma_C \sigma(B).
\]
\end{proposition}

\subsection{Relaxation bound}

For $P,Q \in \R^2$, define function $f_{P,Q}: \R^2 \times \R \times \R \to \R$ as

\[
f_{P,Q} (z,c,s)= \left\|\begin{bmatrix} 
c & -s \\
s& c 
\end{bmatrix} P +z - Q \right\|_2^2\,.
\]

Then, with suitable definition of real constants
$c^-,c^+,s^-,s^+$, depending on the current box $B$, value $f^*(B)$ is equivalent to:
\begin{problem}\label{prob:glores_box_ref}
\begin{subequations}
\begin{alignat}{2}
&\min_{z \in \R^2,\ c,s \in \R,\ S \subset \{1,\ldots,n\}}
&\quad& \sum_{i \in S} \min_{Q \in \Qc} f_{P_i,Q} (z,c,s) \label{eq:obj_fun_box_r}\\
&\text{subject to} &      & |S| = p\\
 & & & z^- \leq z \leq z^+ \label{eq:cond_on_z}\\ & & & c^2+s^2=1 \label{eq:con_cir}
\\ & & &c^-\leq c \leq c^+,\enspace s^-\leq s \leq s^+\,. \label{eq:con_cs}
\end{alignat}
\end{subequations}
\end{problem}
We relax Problem~\ref{prob:glores_box_ref} in the following way:
\begin{itemize}
\item We substitute constraints~\eqref{eq:con_cir},~\eqref{eq:con_cs} with $[c,s] \in \A$,
where $\A$ is a convex polygon that contains the circle arc
$\{[\cos \theta, \sin \theta]\ |\ \theta^- \leq \theta \leq
\theta^+\}$. In our experiments, we employed an isosceles trapezoid,
as illustrated in Figure~\ref{fig:arc_bound_trapezoid}.
\item We substitute function $f_{P,Q}$ with an affine underestimator of the same function, namely, a supporting hyperplane of the function at a point within the current box. 
\end{itemize}
\begin{figure}[h]
  \centering
  \includegraphics[width=0.95\columnwidth]{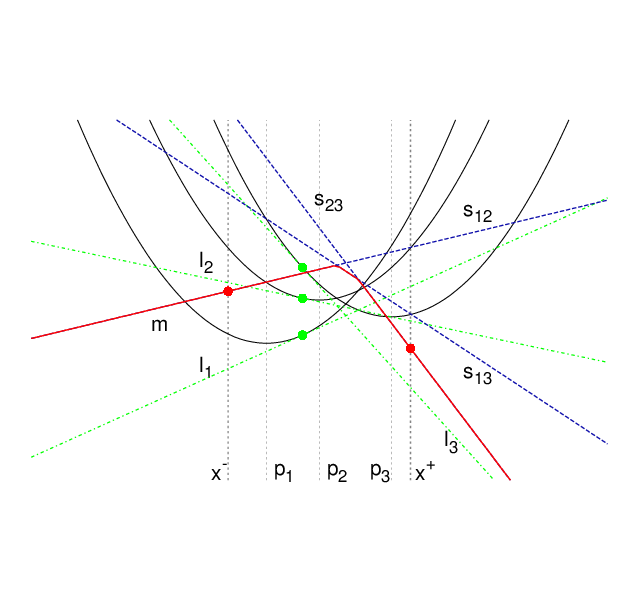}
  \caption{A one dimension example of \emph{relaxation bound}: the squared distance functions from each given point (black); 
     the linearized distance functions (dash-dotted grey lines) at linearization points (green); 
     the concave objective function of the relaxed problem (red).}
  \label{fig:relaxation_bound_schema}
\end{figure}
As we will show in the following, after these two changes, the
modified Problem~\ref{prob:glores_box_ref} consists in the
minimization of a concave function over a polyhedral domain, so that its minimum is
attained at one of the vertices of the domain.
\begin{remark}
To give some intuition on the relaxation bound, we consider the following simple
problem, where $p_1,p_2,p_3 \in \R$ are assigned positions of three
points in a one dimensional domain:
\begin{equation}
  \label{eqn_simple_prob}
  \begin{aligned}
&\min_{x,  S \subset \{1,\ldots,n\}}
& \sum_{i \in S} ((x-p_i)^2 + c_i)
\\ &\text{subject to} &     |S| = 2,\enspace x^- \leq x \leq x^+\,.
\end{aligned}
\end{equation}
Note that the objective function is defined as the minimum of the sum
of convex functions.
Figure~\ref{fig:relaxation_bound_schema} shows the
main idea of the relaxation bound applied to this simple
example. Namely, green dotted
lines $l_1$, $l_2$, $l_3$ represent the supporting hyperplanes of distance functions
$(x-p_i)^2 + c_i$, $i \in \{1,2,3\}$ at the middle point $\frac{x^-+x^+}{2}$.
Blue dotted lines $s_{12}=l_1+l_2$, $s_{13}=l_1+l_3$, $s_{23}=l_2+l_3$ represent
the sums of the couples in set $\{l_1,l_2,l_3\}$, while function
$m=\min \{s_{12},s_{13},s_{23}\}$ is a lower bound of the objective function
in~\eqref{eqn_simple_prob}. Note that function $m$ (represented in
red) is piecewise-linear concave so that
its minimum is achieved at an extremal point of interval $[x^-,x^+]$, that
is, $\min_{x^- \leq x \leq x^+}
m(x)=\min\{m(x^-),m(x^+)\}$. 
\end{remark}
Now, let $\A$ be a convex polygon such that
$\{[\cos \theta, \sin \theta]\ |\ \theta^- \leq \theta \leq \theta^+\} \subset \A$
and let $z_0 \in \R^2$, $c_0,s_0 \in \R$ be assigned parameters. Define
\[
\hat f_{P,Q}(z,c,s) = \nabla f_{P,Q} (z_0,c_0,s_0) [z-z_0,c-c_0,s-s_0]^T
+ f_{P,Q}(z_0,c_0,s_0)\,,
\]
and
\[
\hat f_{P}(z,c,s)= \min_{Q \in \Qc} \hat f_{P,Q} (z,c,s)\,.
\]
Let us consider the following problem.
\begin{problem}\label{prob:non_so_cheap}
\begin{subequations}
\begin{alignat}{2}
&\min_{z \in \R^2,\ c,s \in \R,\ S \subset \{1,\ldots,n\}}        & \sum_{i \in S} \hat f_{P_i} (z,c,s)   \label{eq:obj_fun_chb}\\
&\text{subject to} &      |S| = p \label{eq:card_cons_chb} \\
& & [c,s] \in \A,\enspace z^- \leq z \leq z^+\,, \label{eq:rel_condbis}
\end{alignat}
\end{subequations}
\end{problem}
We prove the following proposition, which leads to the definition of the relaxation bound.
\begin{proposition}
\label{prop:relbound}
Problem (\ref{prob:non_so_cheap})
is a relaxation of Problem~\ref{prob:glores_box_ref}.
\end{proposition}
\begin{proof}
For any $P \in \Pc$, $Q \in \Qc$, $z \in \R^2$, $c,s \in \R$,
$\hat f_{P,Q} (z,c,s) \leq f_{P,Q} (z,c,s)$, since function $f_{P,Q}$
is convex (being the squared norm of a function linear in $z,c,s$) and
$\hat f_{P,Q}$ is its supporting hyperplane at point $ (z_0,c_0,s_0)$.
Hence, for any  $z \in \R^2$, $c,s \in \R$, $S \subset \{1,\ldots,n\}$,
objective function~\eqref{eq:obj_fun_chb} is not larger than ~\eqref{eq:obj_fun_box_r}.
Further,
conditions~\eqref{eq:cond_on_z}~\eqref{eq:con_cir},~\eqref{eq:con_cs}
imply~\eqref{eq:rel_condbis} by the definition of $\A$ so that the feasible region of Problem~\ref{prob:glores_box_ref} is a subset of
the feasible region of Problem~\ref{prob:non_so_cheap}.
\end{proof}
Note that, in our experiments, we have always chosen $(z_0,c_0,s_0)$ equal
to the center of the current box. As a consequence of Proposition \ref{prop:relbound}, the optimal value
of Problem (\ref{prob:non_so_cheap})
is a lower bound for the optimal value of Problem~\ref{prob:glores_box_ref}.
Thus, next step is to prove that the optimal value of Problem  (\ref{prob:non_so_cheap})
can be efficiently computed.
This is stated in the following proposition.
\begin{proposition}
Let $V(B) \subset \R^4$ be the finite set of vertices of polytope
$[z^-,z^+] \times \A$ (we put in evidence the dependency of this set on the current box $B$). Then,
the solution of Problem~\ref{prob:non_so_cheap}
is given by
\[
\min_{x \in V(B)} \sum_{i=1}^{p} \{\min_{Q \in \Qc} \hat f_{P,Q} (x)\ |\ P \in \Pc\}_{(i)} \,.
\]
\end{proposition}

\begin{proof}
Define function $\tilde f: \R^2 \times \R \times \R \to
\R$ as
\begin{equation}
\label{eq:obj_subpr}
\tilde f(z,c,s)= \min_{S \subset \{1,\ldots,n\},|S|=p}         \sum_{i
  \in S} \min_{Q \in \Qc} \hat f_{P_i,Q} (z,c,s) \,.
\end{equation}
Note that function $\tilde f$ is concave. Indeed:
\begin{itemize}
\item each function $\hat{f}_{P_i,Q}$ is linear;
\item $\min_{Q \in \Qc} \hat f_{P_i,Q}$ is concave since it is the minimum of a finite number of linear functions;
\item for each set $S$, we have that $\sum_{i\in S} \min_{Q \in \Qc} \hat f_{P_i,Q}$ is concave since it is a sum of concave functions;
\item finally, $\tilde{f}$ is concave since it is the minimum of a finite set of concave functions (obtained by all possible subsets
$S\subseteq \{1,\ldots,n\}$ with cardinality $p$).
\end{itemize} 
Function $\tilde f$ is easily computed by calculating all values of
$\min_{Q \in \Qc} \hat f_{P_i, Q} (z,c,s)$, for $i \in \{1,\ldots,n\}$ and then summing the $p$
smallest values, namely
$\tilde f(z,c,s)=\sum_{i=1}^{p} \{\min_{Q \in \Qc} \hat f_{P,Q} (z,c,s)\ |\ P \in \Pc\}_{(i)}$.
Further,
Problem~\ref{prob:non_so_cheap} can be reformulated as
\begin{problem}\label{prob:non_so_cheap_ref}
\begin{subequations}
\begin{alignat}{2}
&\min_{z \in \R^2,\ c,s \in \R}        & \tilde f (z,c,s)  \\
&\text{subject to} &   [c,s] \in \A,\enspace z^- \leq z \leq z^+\,. \label{eq:rel_cond}
\end{alignat}
\end{subequations}  
\end{problem}
Note that the feasible region is a polytope, while, as already commented, the objective function is concave.
Since the minimum of a concave function over a polytope is attained at one of
its vertices (for instance, see Property~12, page~58
of~\cite{horst2013handbook}), the solution of
Problem~\ref{prob:non_so_cheap_ref} is obtained by evaluating $\tilde
f$ at each vertex of $[z^-,z^+] \times \A$ and taking the minimum. 
\end{proof}
In the following, for a box $B$ we will set
\[
\phi_R(B)=\min_{x \in V(B)} \sum_{i=1}^{p} \{\min_{Q \in \Qc} \hat f_{P,Q} (x)\ |\ P \in \Pc\}_{(i)}\,.
\]
As shown in the following proposition, the
computation of $\phi_R$ has the same time-complexity of $\phi_C$.
\begin{proposition}
Function $\phi_R$ can be computed with time-complexity $O(n m |V(B)|)$.
\end{proposition}
\begin{proof}
Function $\phi_R$ is evaluated by computing $\hat f_{P,Q} (z,c,s)$ for
all $P \in \Pc$, $Q \in \Qc$ and $(z,c,s) \in V(B)$.
\end{proof}
Since one can fix the cardinality of $V(B)$, the time-complexity of the
computation of $\phi_R$ with respect to the $m$ and
$n$ is the same as the cheap bound $\phi_C$. However, in
practice, computing $\phi_R$ takes more time than $\phi_C$
because of the computation of the gradient $\hat f$
and the necessity of iterating on all vertices of $V(B)$. Anyway, for
small boxes, $\phi_R$ is a much more accurate lower bound than
$\phi_C$. This will be theoretically proved in the following theorem.
Set $\A$ can be conveniently defined as an isosceles trapezoid as
displayed in Figure~\ref{fig:arc_bound_trapezoid}. With this choice,
the following theorem, whose proof is given in Section~\ref{sec_proof}, shows that quantity $|f^*(B) -
\phi_R(B)|$ is bounded by a term proportional to $\sigma(B)^2$, i.e., the relaxation bound is a second-order one, while we proved in Proposition \ref{prop:cheap_error} that the cheap bound is just a first-order one.
\begin{theorem}\label{prop:relaxation_error}
Under the assumption that convex polygon $\A$ is an isosceles
trapezoid (as in Figure~\ref{fig:arc_bound_trapezoid}), there exists a constant $\Gamma_R > 0$, dependent on problem data, such that, for any box $B \in \mathcal{B}$, with $\theta^+ - \theta^- < \frac{\pi}{2}$,
\begin{gather*}
| f^*(B) - \phi_R(B) | \leq \Gamma_R \sigma(B)^2.
\end{gather*}
\end{theorem}
\begin{figure}[th]
  \centering
  \includegraphics[width=0.95\columnwidth]{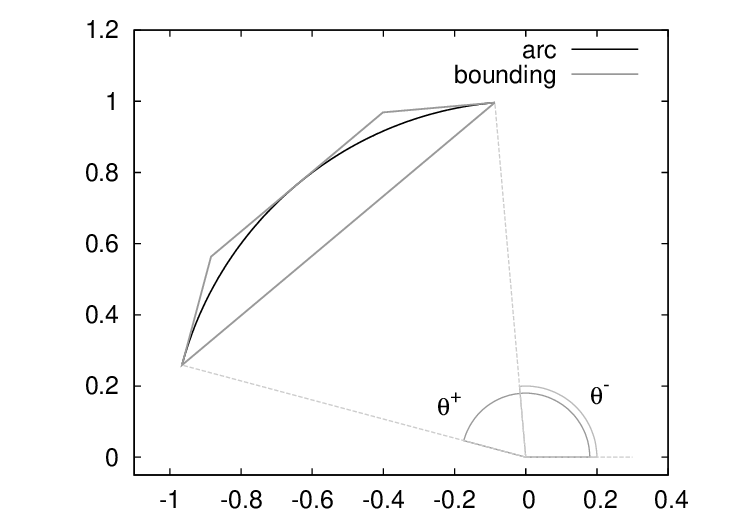}
  \caption{An example of a convex isosceles trapezoid bounding an arc
    used in the computation of the relaxation bound.}
  \label{fig:arc_bound_trapezoid}
\end{figure}
\section{Algorithm description}
\label{sec:algorithm}
A brute-force evaluation of bound $\phi_C$ in~\eqref{eqn_cheap_b} requires the computation of
all $nm$  distances $d_{\min}(B,P,Q)$ for $P \in \Pc$ and $Q \in \Qc$, for all
boxes $B$ encountered during the BnB algorithm. The
computation of $\phi_C$ can be made more efficient with the following
procedure. The main idea is that, if $B' \subset B$ is a box obtained by splitting
$B$, the evaluation of $\phi_C(B')$ is simplified by taking into
account the information already gained in evaluating $\phi_C(B)$.
We first need to introduce some operations on ordered lists.
Let $\Lc = (\R \times \{1,\ldots,m\})^*$ (where $*$ is the Kleene star~\cite{HopcroftMotwaniUllman06}) be the set consisting of all ordered pairs of form $(d,i)$. 
For $L \in \Lc$, we denote by $\first(L)$ the first pair of $L$
and by $\rest(L)$ the list obtained from $L$ by removing its first
element. If $L$ is empty, $\first(L)$ returns $\varnothing_c=(-1,0)$, that denotes the empty pair.
Further, for any pair $c=(d,i)$, we denote by $\add(L,c)$, the list
obtained by adding $c$ to the list, in such a way that the pairs of
$\add(L,c)$ are ordered in ascending order with respect to the first
element of the pair.
We associate to each box $B \in \mathcal{B}$ and each $P \in \Pc$ a list $L_{B,P} \in
\Lc$. The elements $(d,i)$ of $L_{B,P}$ will have the
following meaning: $i$ will be the
index of a point in set $\Qc$ and $d$ will be a lower bound for $d_{\min}(B,P,Q_i)$.
Further, we associate to each $B$ and $P$ a term $U_{B,P} \in \R$,
that will represent an upper bound for $\min_{Q \in \Qc} d_{\max}(B,P,Q)$.
At the beginning, Algorithm~\ref{alg_init} is applied
to each $P \in \Pc$, where $B_0$ is the box corresponding to the complete domain.
\begin{algorithm}
\caption{Initialization}
\label{alg_init}
\begin{algorithmic}[1]
  \Function {init} {$B_0,P,\Qc$}
\State $U_{B_0,P}=\infty$
\For {$j=1$ to $m$}
\State $\ell=d_{\min}(B_0,P,Q_j)$
\State $\Call{add}{L_{B_0,P}, (\ell,j)}$
\State $U_{B_0,P}=d_{\max}(B_0,P,Q_j) \wedge U_{B_0,P}$
\EndFor
\State \Return $L_{B_0,P}$, $U_{B_0,P}$
\EndFunction
\end{algorithmic}
\end{algorithm}
Namely, for each $P \in \Pc$, $L_{B_0,P}$ is initialized by adding all
pairs $\{d_{\min}(B_0,P,Q_i)\ |\ i \in \{1,\ldots, m\}\}$, corresponding to all minimum
distances between points in $\Pc$, transformed with all possible
transformations corresponding to parameters in
$B_0$, and the elements of $\Qc$. Further $U_{B_0,P}$ is initialized to
$\min \{d_{\max}(B_0,P,Q_i)\ |\ i \in \{1,\ldots, m\}\}$, that is the minimum of the upper
bounds of the distances between the transformed point $P$ and point
set $\Qc$.
\begin{figure}[th]
    \centering
    \includegraphics[width=0.68\columnwidth]{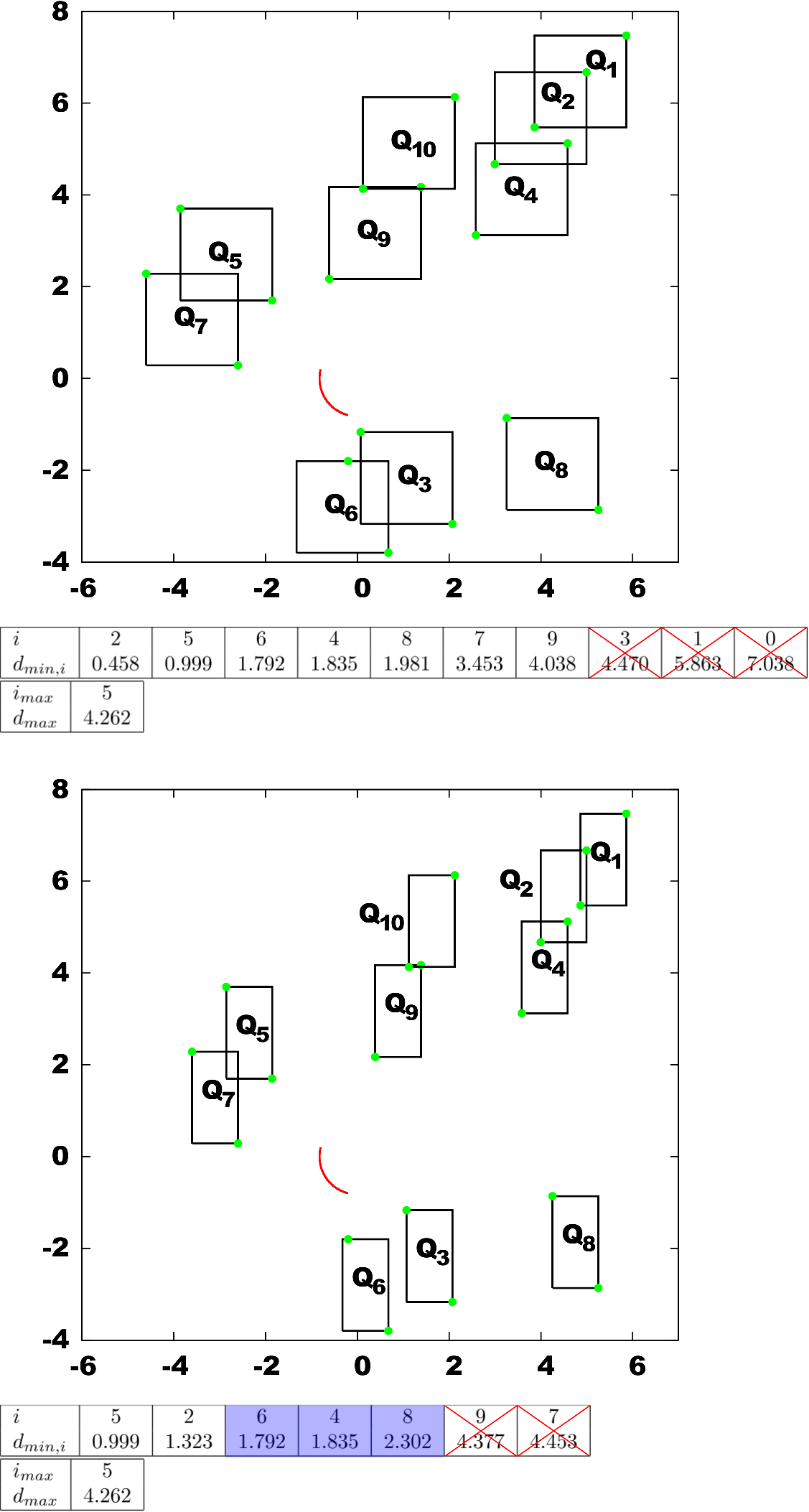}
    \caption{An example illustrating the operations performed in
      Algorithm~\ref{alg_lb}, for a source point $P$ and a set of
      destination points $Q_i$, with $i \in \{1,\ldots,10\}$. The top figure considers
box $B = [0,2] \times [-2,0] \times [\pi/2, \pi]$ and displays 
arc $\{R(\theta) P\ |\ \theta \in
[\pi/2,\pi]\}$ in red and rectangles
$\{x \in \R^2\ |\ Q_i-[2,0]^T \leq x
\leq Q_i-[0,-2]^T\}$, with $i \in \{1,\ldots,n\}$.
The bottom subfigure represents the same geometrical objects, associated to the child box
$B' = [0,1] \times [-2,0] \times [\pi/2, \pi]$.
      The nearest and farthest points from the arc of each rectangle are colored in green. 
      The table below the top figure reports quantities $L_{B,P},
      U_{B,P}$ and the table below the bottom figure reports
      $L_{B,P}, L_{B',P}$.
   }
    \label{fig:cheap_bound_queue_example}
  \end{figure}
Then, each time a box $B'$ is obtained by splitting a box $B$,
Algorithm~\ref{alg_lb} is applied to $B'$, for each $P \in \Pc$.
\begin{algorithm}
\caption{Lower bound computation}
\label{alg_lb}
\begin{algorithmic}[1]
\Function {update}{$B'$, $B$, $L_{B,P}$, $U_{B,P}$, $\Qc$}
\State $(d_{\text{old}},i)=\first(L_{B,P})$ \label{alg_lb_begin}
\State $L_{B,P}=\rest(L_{B,P})$
\State $m=d_{\min}(B',P,Q_i)$ \label{alg_comp_m} 
\State $U_{B',P}= d_{\max}(B',P,Q_i)$ \label{alg_lb_init_end} \label{alg_first_upd_u}
\State $\add{L_{B',P},(m,i)}$
\State $(d_{\text{old}},i)=\first(L_{B,P})$  \label{alg_lb_first_while}
\State $L_{B,P}=\rest(L_{B,P})$
\While {$d_{\text{old}} \leq m$ and $(d_{\text{old}},i) \neq \varnothing_c$} \label{alg:cond_first_while}
\State $d=d_{\min}(B',P,Q_i)$
\State $m= m \wedge d$
\State $U_{B',P}= U_{B',P} \wedge d_{\max}(B',P,Q_i)$ \label{alg_sec_upd_u}
\State add $(d,i)$ to $L_{B',P}$ \label{alg_ins_while1}
\State $(d_{\text{old}},i)=\first(L_{B,P})$
\State $L_{B,P}=\rest(L_{B,P})$
\EndWhile \label{alg_lb_first_while_end}
\While {$U_{B',P} > d_{\text{old}}$ and $(d_{\text{old}},i) \neq \varnothing_c$} \label{alg_lb_sec_while}
\State add $(d_{\text{old}},i)$ to $L_{B',P}$ \label{alg_ins_while2}
\State $(d_{\text{old}},i)=\first(L_{B,P})$
\State $L_{B,P}=\rest(L_{B,P})$
\EndWhile \label{alg_lb_sec_while_end}
\State \Return {$L_{B',P}$, $U_{B',P}$}
\EndFunction
\end{algorithmic}
\end{algorithm}
In lines~\ref{alg_lb_begin}--\ref{alg_lb_init_end} the first element $(d_{\text{old}},i)$
of $L_{B,P}$ is removed from $L_{B,P}$ and it is set
$m=d_{\min}(B',P,Q_i)$, $U_{B',P}=d_{\max}(B',P,Q_i)$. Variables $m$ and $U_{B',P}$
represent the minimum values of $d_{\min}(B',P,Q_i)$ and
$d_{\max}(B',P,Q_i)$, respectively. Note that, at each iteration, $m$ and $U_{B',P}$ will denote the minimum 
values of $d_{\min}(B',P,Q)$ and
$d_{\max}(B',P,Q)$ 
among all points $Q$ already processed, respectively. 
Then, pair
$(m,i)$ is added to list $L_{B',P}$ associated to the new box $B'$, while
the next pair $(d_{\text{old}},i)=\first(L_{B,P})$ is removed from $L_{B,P}$.
In lines~\ref{alg:cond_first_while}--\ref{alg_lb_first_while_end},
a while cycle is iterated until
 either the list is empty, in which case all points have been processed,
or $d_{\text{old}} \leq m$, in which case  the computation
of $d_{\min}(B',P,Q_i)$ is not necessary since it would not alter the
value of $m$. Note that this is also true for the remaining pairs
$(d,i)$ of
$L_{B,P}$, since the list is ordered in ascending order of $d$.
If the stopping conditions are not fulfilled, 
quantities $d=d_{\min}(B',P,Q_i)$, $d_{\max}(B',P,Q_i)$ are recomputed
(updating $m$ if required) and the updated
couples $(d,i)$ are added to $L_{B',P}$.
Finally, in lines~\ref{alg_lb_sec_while}--~\ref{alg_lb_sec_while_end}
all remaining pairs $(d_{\text{old}},i)$ of $L_{B,P}$ are processed. All pairs
such that $d_{\text{old}}<U_{B',P}$ are added to $L_{B',P}$. Note that pairs
$(d_{\text{old}},i)$ that are not added to $L_{B',P}$ correspond to points
$Q_i$ that have a distance to $P$ that is greater than the current
value of $U_{B',P}$. Hence, they are not added to $L_{B',P}$ since
they cannot be the element of $\Qc$ with minimum distance to
$P$ for all transformations belonging to $B'$ (or any of its
subsets). In this way, list $L_{B',P}$ can be smaller than $L_{B,P}$,
speeding up all subsequent iterations of the algorithm in boxes
obtained by recursively dividing $B'$.
The following definition states the properties that quantities $L_{B,P}$ and
$U_{B,P}$ must satisfy for the correct application of Algorithm~\ref{alg_lb}.
\begin{definition}
$L_{B,P}$ and $U_{B,P}$ are consistent if the following
conditions are satisfied:
\begin{enumerate}
\item For each $(d,i) \in L_{B,P}$, $d_{\min}(B,P,Q_i) \geq d$;
\item $U_{B,P} \geq \min_{i \in \{1,\ldots,m\}}
d_{\max}(B,P,Q_i)$;
\item For each $i \in \{1,\ldots,n\}$ such that there does not exist any $d \in
\R$ such that $(d,i) \in L_{B,P}$, $d_{\min}(B,P,Q_i)\geq d^*$, where $(d^*,i^*)=\first(L_{B,P})$;
\item $d^*=\min_{Q \in \Qc} d_{\min}(B,P,Q)$,  where $(d^*,j)=\first(L_{B,P})$.
\end{enumerate} 
\end{definition}
\begin{remark}
In order to better understand the queue update performed by
Algorithm~\ref{alg_lb}, we present an example consisting of a
source point $P$ and a set of destination points $Q_i$,
$i \in \{1,\ldots,10\}$, illustrated in
Figure~\ref{fig:cheap_bound_queue_example}.
Define box $B = [0,2] \times [-2,0] \times [\pi/2, \pi]$.
The top subfigure displays arc $\{R(\theta) P\ |\ \theta \in
[\pi/2,\pi]\}$, associated to point $P$ and rectangles
$\{x \in \R^2\ |\ Q_i-[2,0]^T \leq x
\leq Q_i - [0,-2]^T\}$ associated to points $Q_i$, $i \in \{1,\ldots,n\}$.
The table below the top subfigure of
Figure~\ref{fig:cheap_bound_queue_example} reports the elements of the
sorted queue $L_{B,P}$ and the scalar $U_{B,P}$, for a consistent
couple $L_{B,P}$, $U_{B,P}$. Namely, queue
$L_{B,P}$ contains couples $(d_{\min,i},i)=(d_{\min}(B,P,Q_i),i)$, sorted in ascending
order with respect to the the first element, while $U_{B,P}= \min_{i} d_{\max}(B,P,Q_i)$. 
The couples $(d_{\min,i},i)$ such that $d_{\min, i} > d_{\max}$ are not
present in queue $L_{B,P}$ and are represented as crossed boxes. 
The bottom subfigure in Figure~\ref{fig:cheap_bound_queue_example}
represents the same geometrical objects, associated to the child box
$B' = [0,1] \times [-2,0] \times [\pi/2, \pi]$ obtained after
splitting $B$ along the first coordinate.
The table below the bottom subfigure of
Figure~\ref{fig:cheap_bound_queue_example} shows queue $L_{B',P}$ and
scalar $U_{B',P}$, obtained as the output of Algorithm~\ref{alg_lb}.
In particular, the white boxes represent the couples added after the
execution of the while cycle in
lines~\ref{alg:cond_first_while}--\ref{alg_lb_first_while_end}, in
which the value of the minimum distance terms $d_{\min,i}$ are updated.
The blue boxes represent the couples added in the second while cycle
in lines~\ref{alg_lb_sec_while}--\ref{alg_lb_sec_while_end}, in
which the values of $d_{\min,i}$ are not updated, while
the boxed cells represent the discarded couples among those
originally present in $L_{B,P}$.
\end{remark}
Note that $L_{B_0,P}$ and $U_{B_0,P}$, as defined by the initialization
procedure, are consistent.
As a consequence, the following proposition proves, by induction, that
consistency holds at each iteration of the main BnB algorithm.
\begin{proposition}
In Algorithm~\ref{alg_lb}, if input variables $L_{B,P}$ and $U_{B,P}$ are
consistent, then output variables $L_{B',P}$ and $U_{B',P}$ are also
consistent.
\end{proposition}
\begin{proof}
$L_{B',P}$ satisfies condition 1). In fact, if pair $(d,i) \in
L_{B',P}$ is inserted in line~\ref{alg_ins_while1}, then, actually
$d=d_{\min}(B',P,Q_i)$. If it is inserted in
line~\ref{alg_ins_while2}, $d$ is not recomputed, so that
$d \leq d_{\min}(B,P,Q_i) \leq d_{\min}(B',P,Q_i)$, where the first
inequality holds by assumption and the second since $B' \subset B$.
Moreover, $U_{B',P}$ is updated for the last time either in
line~\ref{alg_first_upd_u} or line~\ref{alg_sec_upd_u}. In both cases
there exists $i \in \{1,\ldots,n\}$ such that
$U_{B',P} = d_{\max}(B',P,Q_i)$, so that property 2) holds.
To prove 3), assume by contradiction that $d_{\min}(B',P,Q_i) <
d^*$. This implies that
$d_{\min}(B',P,Q_i) \leq d^*$, so that $(d,i)$ must have been inserted
in $L_{B',P}$ in line~\ref{alg_ins_while1} and, being $d < d^*$, $(d^*,i^*) \neq
\first(L_{B',P})$ (being list $L_{B',P}$ ordered) which is a contradiction.
Finally, to prove 4), assume by contradiction that there exists $i \in
\{1,\ldots,n\}$ such that $d^*>d_{\min}(B',P,Q_i)$. By point 3) this
implies
that $L_{B,P}$ contains a couple of form $(d,i)$.
Anyway, no couple of the form $(d,i)$ is added in
line~\ref{alg_ins_while1}, because this would imply that $d^* \leq d$.
This implies that such couple is added in line~\ref{alg_ins_while2},
so that $(d,i) \in L_{B,P}$, that is $(d,i)$ belongs also to the queue
of the parent box $B$.  Note that quantity $m$ at
line~\ref{alg:cond_first_while} is such that $m \geq d^*$. This
implies that the exit condition for the while-loop
in~\ref{alg:cond_first_while} is satisfied for $d_\text{old}=d$,
contradicting the previously stated fact that no couple of form
$(d,i)$ is added to list $L_{B,P}$ in that cycle.
\end{proof}
If $\{L_{B,P}\ |\ P \in \Pc\}$ are consistent bounds for $B$, then $\phi_C(B)$ can be
computed by summing the first elements of lists $\{L_{B,P}\ |\ P \in
\Pc\}$ as stated in Algorithm~\ref{alg_phic_comp}.
\begin{algorithm}
\caption{Computation of $\phi_C$ from $L_{B,P}$}
\label{alg_phic_comp}
\begin{algorithmic}[1]
\Function {ComputeLBC}{$\Pc$, $\{L_{B,P}\ |\ P \in \Pc\}$,$p$}
\State $b=0$
\For {$P \in \Pc$}
\State $(b_P,i)= \first(L_{B,P})$
\EndFor
\State \Return {$\sum_{i=1}^p \{b_P\ |\ P \in \Pc\}_{(i)}$}
\EndFunction
\end{algorithmic}
\end{algorithm}
\begin{algorithm}
\caption{Computation of $\phi_R$ for $B$}
\label{alg_phir_comp}
\begin{algorithmic}[1]
\Function {ComputeLBR}{$\Pc$, $B$, $\{L_{B,P}\ |\ P \in \Pc\}$, $p$}
\State Write $B$ as $\{(z,\theta)\ |\ z^-\leq z \leq z^+, \theta^- \leq
\theta \leq \theta^+\}$
\State Let $\A$ be the vertices of a convex polygon containing arc 
$\{[\cos \theta,  \sin \theta]\ |\ \theta^- \leq \theta \leq \theta^+\}$
\State Set $\V=\{z^-,z^+\} \times \A$
\State $\ell=\infty$ 
\For {$v \in \V$}
\State Set $X=\varnothing$
\For {$P \in \Pc$}
\State $m=\infty$
\For {$(d,i) \in L_{B,P}$}
\State $m=m \wedge \hat f(P,Q_i)$
\EndFor
\State Set $X=X \cup \{m\}$
\EndFor
\State $\ell=\ell \wedge \sum_{i=1}^p \{X\}_{(i)}$
\EndFor
\State \Return {$\ell$}
\EndFunction
\end{algorithmic}
\end{algorithm}
Algorithm \ref{alg_phir_comp} details the computation of the relaxation bound $\phi_R$, while
the overall implementation of the BnB method is reported in
Algorithm~\ref{alg_bnb_imp}. Here, $\delta$ is a threshold value for
using lower bound function $\phi_R$. Namely, $\phi_C$ is always
computed, since its computational cost is very low, while $\phi_R$ is
computed only for sufficiently small boxes, where it is more precise
than $\phi_C$, as a consequence of Theorem~\ref{prop:relaxation_error}.
\begin{algorithm}
\caption{Actual implementation of BnB algorithm}
\label{alg_bnb_imp}
\begin{algorithmic}
\State Input: $\Pc$, $\Qc$, $R$: data for Problem~\ref{prob:glores_box}
\State $\epsilon$: solution tolerance
\State $\delta$: tolerance for lower bound selection
\State Output: $x^*$: optimal solution 
\begin{enumerate}
\State Let $\zeta$ be a list of boxes and initialize $\zeta=\{R\}$.
\State For any $P \in Pc$, set $(L_{R,P},U_{R,P})=\Call{init}{R,P,\Qc}$
\State Set $UB=f(r(R))$, and $x^*=r(R)$.
\State If $\zeta=\varnothing$, stop. Else set $\delta_{\max} =
  \max \{ \delta(\eta)\ |\ \eta \in \zeta \}$.
\State Select a box $\eta \in \zeta$, with $\delta(\eta)=\delta_{\max}$
  and split it into two equal smaller sub-boxes $\eta_{1}$, $\eta_{2}$ along
  the dimension of maximum length.
\State Delete $\eta$ from $\zeta$ and add $\eta_{1}$ and $\eta_{2}$ to $\zeta$.
\State Update $UB=\min\{UB, f(r(\eta_{1})), f(r(\eta_{2}))\}$.
If $UB= f(r(\eta_j))$ with $j \in \{1,2\}$, set
$x^*= r(\eta_{j})$.
\For{$j=1,2}$
\State for any $P \in \Pc$, set
$(L_{\eta_j,P},U_{\eta_j,P})=\Call{update}{\eta_j,\eta,P,\Qc,L_{\eta,P},U_{\eta,P}}$
\State $\phi_C=\Call{computeLBC}{\Pc,\{L_{\eta_j,P}\ |\ P \in \Pc\}}$
\If{ the largest dimension of $\eta$ is lower than $\delta$}
\State $\phi_R=\Call{computeLBR}{\Pc,\eta_j,\{L_{\eta_j,P}\ |\ P \in \Pc\},p}$
\Else
\State $\phi_R=\infty$
\EndIf
\State $\phi_L(\eta_j)=\phi_R \wedge \phi_C$
\EndFor
\State For all $\kappa \in \zeta$, if $\phi_L(\kappa) (1+\epsilon) \geq UB$ set $\zeta=\zeta \setminus \kappa$.
\State Return to step 3.
\end{enumerate}
\end{algorithmic}
\end{algorithm}
\subsection{Comparison with the method
  presented in~\cite{YangLiCampbellJia16}}
\label{subs_comp}
At this point, we can compare in more detail our method with the one
proposed in~\cite{YangLiCampbellJia16}, which is probably the
reference in literature that bears more similarities with our approach,
even if~\cite{YangLiCampbellJia16} considers the 3D registration
problem, while we focus on the planar case.
In fact, both methods are based on a BnB search on the rigid
transformation parameters domain. Moreover, the cheap bound that we
defined in this paper is similar to the lower bound estimate
defined in~\cite{YangLiCampbellJia16}.
Moreover, with respect to~\cite{YangLiCampbellJia16}, the present
paper introduces the following features:
\begin{itemize}
\item We introduce the relaxation bound, which is much more accurate
  than the cheap bound in small boxes (see
  Proposition~\ref{prop:cheap_error}
  and Theorem~\ref{prop:relaxation_error}). This second bound allows to close
  the gap between the upper and lower bound in the BnB algorithm
  more efficiently than the method presented
  in~\cite{YangLiCampbellJia16}. 
  With the numerical experiment presented in Section~\ref{sec:results}, we will show that the relaxation bound
  is of fundamental importance for the efficiency of the algorithm.
\item We introduce the queue-based algorithm, presented in
  Section~\ref{sec:algorithm}, that reduces the number of
  point-to-point distances to be computed during the BnB recursion. In
  Section~\ref{sec:results}, we will show that also this improvement largely
  affects the overall efficiency of the algorithm.
  \end{itemize}
Moreover, note that in the method presented
in~\cite{YangLiCampbellJia16}, the computation of the objective
function is based
on a precomputed distance transformation on a finite grid. This
approach considerably speeds up the computation of the objective
function but also introduces a significant error in its evaluation
that affects the overall accuracy of the algorithm.
In our method, we do not use a precomputed distance transformation,
this reduces the speed in the evaluation of the objective function,
but does not affect the accuracy of the final result.
\subsection{Extension to 3D case}
\label{sec:3Dcase}
The algorithm can be extended to the 3D case. Note that a 3
dimensional rigid transformation can be described by 6 parameters (3
rotations and 3 translation variables), while 3 parameters are
sufficient to represent a rigid planar transformation. 
The larger dimension requires the development of a very
efficient lower bound estimator, in order to avoid evaluating an
excessively high number of nodes in the BnB algorithm.
One possibility for such a bound consists in the extension of
the relaxation bound presented in this paper to the 3D case, using
quaternions to parameterize rotations. We are currently working in
this direction. Anyway, the 3D case is significantly different from
the planar one and is outside the scope of this paper.
\section{Proofs}
\label{sec_proof}
\subsection{Computation of $d_{\min}$ and $d_{\max}$}
\label{sec:dmindmax}
\begin{figure}[th]
  \centering
  \includegraphics[width=0.95\columnwidth]{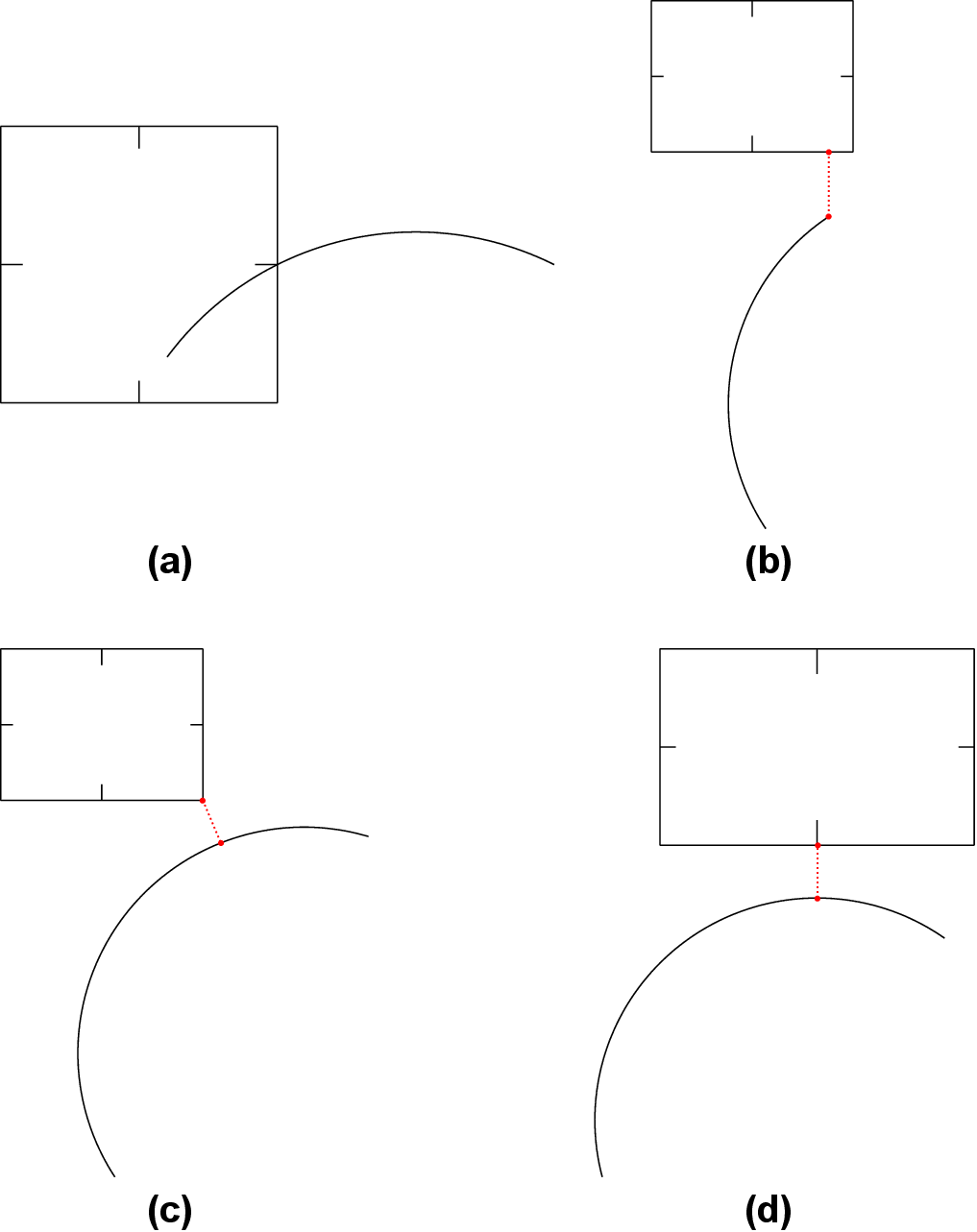}
  \caption{Four configurations of rectangle and arc relevant for the computation of minimum distance between the two geometric shapes:
    (a) zero distance with intersection; (b) one arc endpoint is the closest arc point to the rectangle; (c) one rectangle vertex is closest to the arc; 
    (d) the closest points are both internal points of a rectangle edge and the arc, respectively. 
  }
  \label{fig:arc_rectangle_distance_min}
\end{figure}
In order to compute the minimum distance between a circle arc and a rectangle we adopt the following strategy:
first, we check whether the arc-rectangle distance is zero.
We do that by checking if at least one vertex of the rectangle lies outside the circle to which the arc belongs and if at least one edge of the rectangle has a non-empty intersection with the arc (see Figure~\ref{fig:arc_rectangle_distance_min} (a)).
Indeed, if all rectangle vertices lay inside the circle, then the rectangle is contained in the circle and there are no points in common between the arc and the rectangle.
If the distance is greater than zero, we have three possibilities: the
minimum distance can be attained at an extreme point of the arc (see
Figure~\ref{fig:arc_rectangle_distance_min} (b)), at a vertex of the
rectangle and an internal point of the arc (see
Figure~\ref{fig:arc_rectangle_distance_min} (c)) or at an internal
point of an edge of the rectangle and an internal point of the arc
(see Figure~\ref{fig:arc_rectangle_distance_min} (d)). Function
$d_{\max}$ can be computed by taking the maximum of all the distances
between the vertices of the rectangle and the arc.
\subsection{Proof of Proposition~\ref{prop:cheap_error}}
Given a box $B \in \mathcal{B}$, let $\bar z, \bar \theta$ and $\bar
S$ be such that $f^*(B) = \sum_{i \in \bar S} f_{P_i}(\bar z, \bar
\theta)$.
Moreover, let $\tilde S$ be such that  $\sum_{i \in \tilde S}
    d_{\min} (B, P_i) =\sum_{i =1}^p \{d_{\min} (B, P)\ |\ P \in \mathcal{P}\}_{(i)}$.
Then, we can write the error as follows
\begin{gather}
\left| f^*(B) - \phi_C(B) \right| 
= \left| \sum_{j \in \bar S} f_{P_j}(\bar z, \bar \theta) - \sum_{i \in \tilde S}   d_{\min} (B, P_i) \right|\,.\label{eq:cheap_error1}
\end{gather}
By definition of $f^*(B)$, it follows that
\[
\sum_{j \in \bar S} f_{P_j}(\bar z, \bar \theta) \leq
\sum_{j \in \tilde S} f_{P_j}(\bar z, \bar \theta)\,,
\]
so that
\begin{gather}
\left| \sum_{i \in \bar S} f_{P_i}(\bar z, \bar \theta) - \sum_{i
    \in \tilde S}   d_{\min} (B, P_i) \right| \leq 
\left| \sum_{i \in \tilde S}\left( f_{P_i}(\bar z, \bar \theta) -
    d_{\min} (B, P_i) \right)\right|\,. \label{eq:cheap_error2}
\end{gather}
Now, we can estimate from above each difference in~\eqref{eq:cheap_error2} with $\Delta \sigma(B)$, where $\Delta = \max_{i \in \{1, \ldots, p\}} |\nabla f_{P_i}|$ and $\nabla f_{P_i}$ denotes the gradient of $f_{P_i}$, from which we obtain
\[
\left| f^*(B) - \phi_C(B) \right| \leq \sum_{i = 1}^p \Delta \sigma(B) = p \Delta \sigma(B).
\]
\subsection{Proof of Theorem~\ref{prop:relaxation_error}}
Let $B$ be a box in $\mathcal{B}$, if $\theta^+ - \theta^- < \frac{\pi}{2}$, then the maximum distance between the circle arc $\{[\cos\theta, \sin\theta]\ |\ \theta^- \leq \theta \leq \theta^+\}$ and convex isosceles trapezoid $\mathcal{A}$ is given by $\rho\left(1 - \cos\left(\frac{\theta^+ - \theta^-}{2}\right)\right)$, where $\rho$ is the radius of the circle arc.
Indeed, under this hypothesis, the pair of points for which the distance arc-trapezoid is maximized is given by the midpoint of the segment joining $(\cos\theta^-, \sin\theta^-)$ and $(\cos\theta^+, \sin\theta^+)$, that is $\left( \frac{\cos\theta^+ + \cos\theta^-}{2}, \frac{\sin\theta^+ + \sin\theta^-}{2} \right)$, and the mid point of the circle arc $\left( \cos\left( \frac{\theta^+ + \theta^-}{2} \right), \sin\left( \frac{\theta^+ + \theta^-}{2} \right) \right)$.
We can estimate the error by decomposing it into two components as follows
\begin{gather}
\left| f^*(B) - \phi_R(B) \right|
\leq \left| f^*(B) - \phi_\ell(B) \right| + \left| \phi_\ell(B) - \phi_R(B) \right|, \label{eq:error_r}
\end{gather}
where the first term in~\eqref{eq:error_r} represents the linearization error given by $\phi_\ell$, in which Problem~\ref{prob:glores_box_ref} has been relaxed by substituting $f_{P,Q}$ with a linearization, whilst the second term represents the approximation error due to the substitution of the circle arc with isosceles trapezoid $\mathcal{A}$.
Now, by setting
\[
\Delta := \frac{1}{2} \max_{P \in \Pc, Q \in \Qc}\max_{|\alpha| = 2} \max_{x} |D^\alpha f_{P,Q}(x) |,
\]
we can estimate from above the first term of~\eqref{eq:error_r} using Taylor's Theorem for multivariate functions (see, for instance,~\cite{KOENIGSBERGER2002}) and the second one applying the consideration on isosceles trapezoid $\mathcal{A}$ we stated earlier, as follows
\begin{equation}\label{eq:error_r2}
p\Delta \sigma(B)^2 + p\Delta \rho \left(1 - \cos\left(\frac{\theta^+ - \theta^-}{2}\right)\right).
\end{equation}
Since $(\forall x \in \R)\ 1 - \cos(x) \leq \frac{1}{2}x^2$, we can rewrite~\eqref{eq:error_r2} as
\begin{equation}\label{eq:error_r3}
p\Delta \sigma(B)^2 + p\Delta \frac{\rho}{8} \left(\theta^+ - \theta^-\right)^2.
\end{equation}
Now, considering that $\theta^+ - \theta^- \leq \sigma(B)$, we obtain
\[
\left| f^*(B) - \phi_R(B) \right|
\leq \left(1 + \frac{\rho}{8}\right)p\Delta\sigma(B)^2.
\]
\section{Computational results}
\label{sec:results}
This section presents two numerical experiments, the first one is
based on random data, the second one on real data acquisitions.
\subsection{Randomly generated problems}
We generated various random instances of Problem~\ref{prob:glores}. In
each case, set $\Pc=\{P_1,\ldots,P_n\}$ contains $n$ random
points with coordinates uniformly distributed in interval $[-10,10]$.
Point set $\Qc=\{Q_1,\ldots,Q_n\}$ is defined by 
\[
Q_i=T_{P_i} (z,\theta) + (1-o_i) \eta_i + o_i \gamma_i,\quad i \in \{1,\ldots,n\}\,.
\]  
Here, $z \in \R^2$ is a random vector whose components are uniformly distributed in
$[-10,10]$ and $\theta$ is a random angle obtained from a uniform
distribution in $[0,2\pi]$.  Vector
$[o_1,\ldots,o_n]$ is such that it has $\lceil 0.1 n \rceil$
randomly selected components equal to $1$, the others being set to $0$.
For $i \in \{1,\ldots,n\}$, $\eta_i \in \R$ is a random
number
obtained from a gaussian distribution centered at $0$ with standard
deviation $\sigma$, while $\gamma_i$ is uniformly
distributed in interval $[-10,10]$.
Note that, if $o_i=1$, then, with high probability, a very large noise
component $\gamma_i$ is added to $Q_i$, so that $Q_i$ becomes an outlier. 
We set $\epsilon=0.0001$, $p=\lceil 0.8 n \rceil$, $\delta=0.1$ and considered different numbers of points $n$, taken from a set of logarithmically spaced values between
$10$ and $300$, while $\sigma$ varies in set
$\{10^{-4},10^{-3},10^{-2},10^{-1}\}$.
We ran these tests on a 2,6 GHz Intel Core i5 CPU, with 8 GB of
RAM. We implemented Algorithm~\ref{alg_bnb_imp} in C++\footnote{The implementation is available at \url{https://github.com/dlr1516/glores}.} and used Matlab
for plotting the results.  
Figure~\ref{fig:test1_time}
reports the computation time for different values of $\sigma$ as a
function of the number of points $n$. Figure~\ref{fig:test1_nodes}
shows the total number of nodes in the BnB algorithm. Note
that $\sigma$ has a large effect on both the computation time and the
number of nodes. This is related to the fact that larger values of
$\sigma$ cause an higher number of local minima in the objective
function of Problem~\ref{prob:glores}, that determines a larger number of explored nodes.
\begin{figure}[h]
    \centering
    \includegraphics[width=0.95\columnwidth]{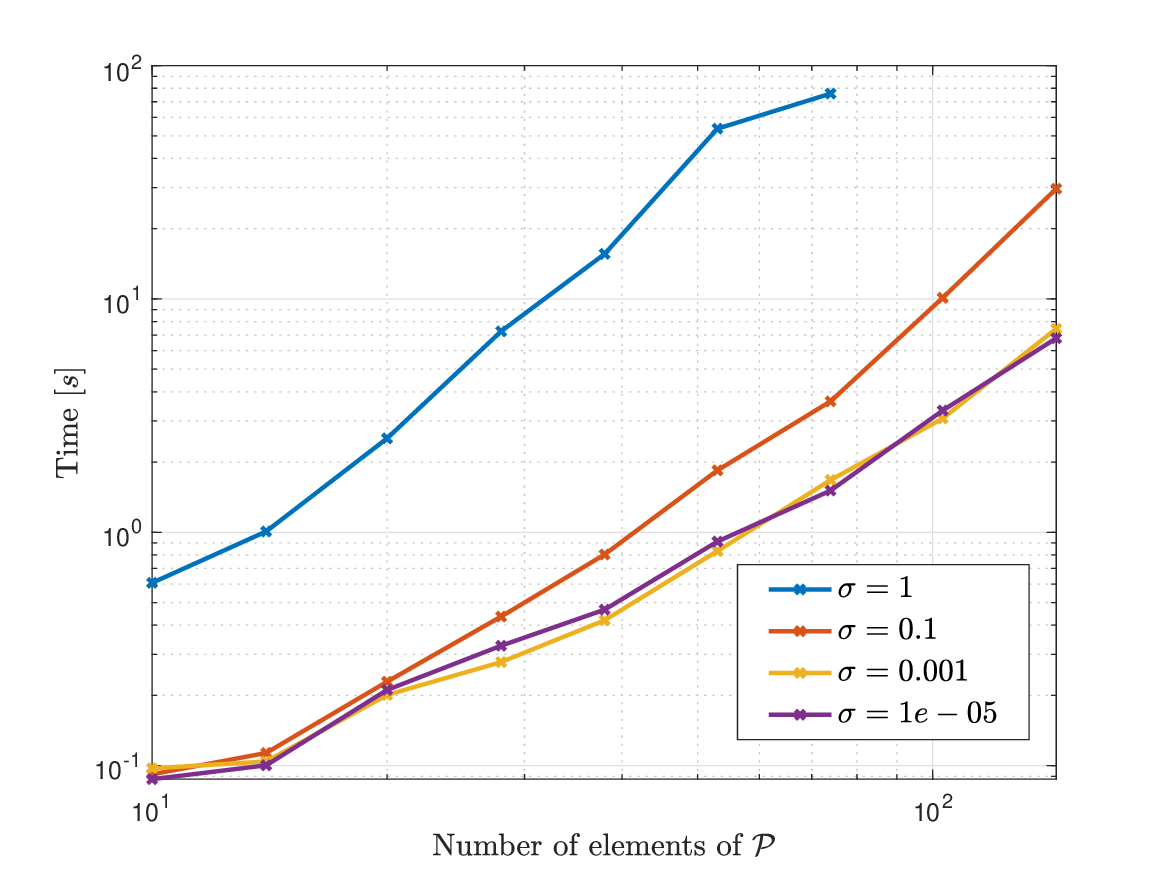}
    \caption{Computation times for randomly generated problems.}
    \label{fig:test1_time}
  \end{figure}  
\begin{figure}[h]
    \centering
    \includegraphics[width=0.95\columnwidth]{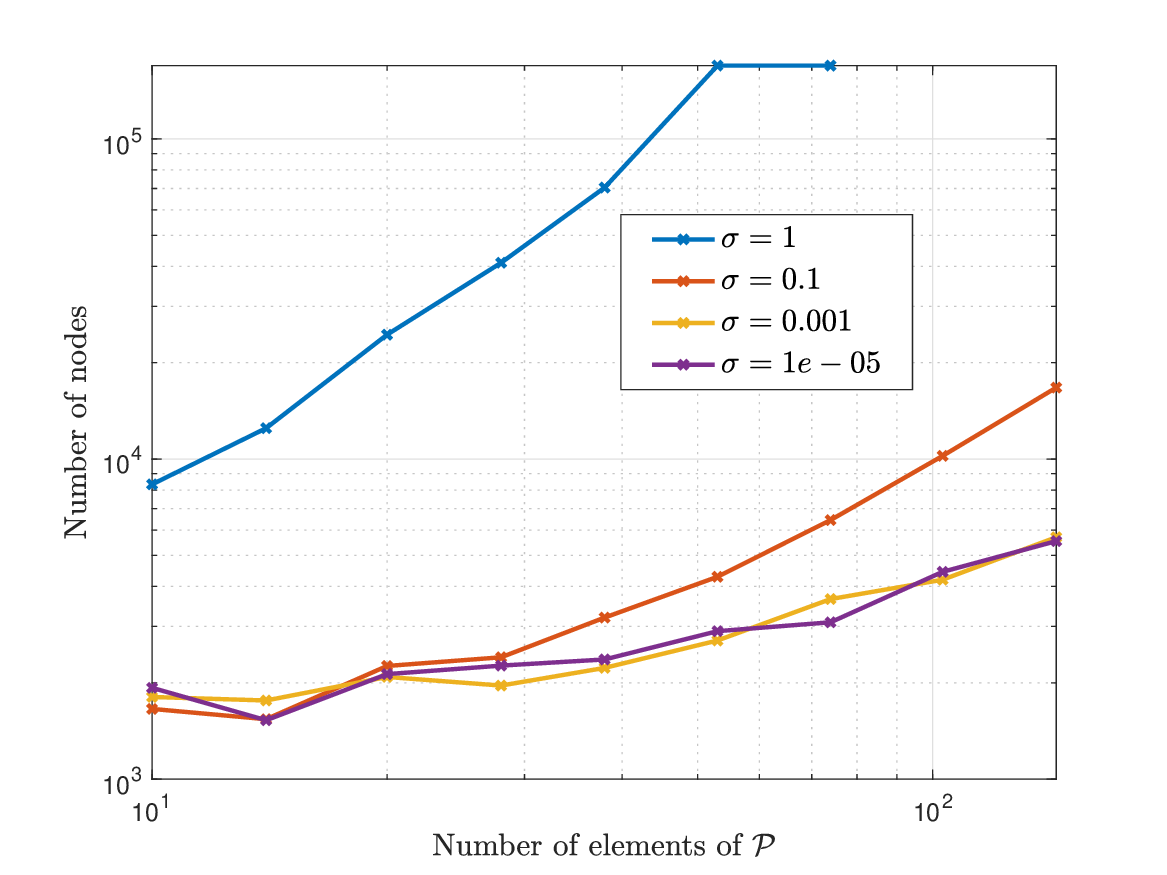}
    \caption{Number of explored nodes for randomly generated problems.}
    \label{fig:test1_nodes}
\end{figure} 
In a second test, we solved random instances of
Problem~\ref{prob:glores}, obtained in the same way as in the first
test, for different values of
the threshold parameter $\delta$ and of standard
deviation $\sigma$.  Figure~\ref{fig:test1_timetr}
reports the computation time for different values of $\sigma$ as a
function of $\delta$, while Figure~\ref{fig:test1_nodestr}
shows the total number of nodes in the BnB algorithm.
These figures suggest that the choice of
$\delta$ is critical to the efficiency of the algorithm. Namely, if
$\delta$ is too small, the relaxation bound is applied only to very
small boxes. In this way, the lower bound of the vast majority of
boxes is computed only with the cheap bound, that is less
efficient (see Proposition~\ref{prop:cheap_error}). In fact, both computational times and total number of nodes increase
very much for lower values of $\delta$.
On the other
hand, if $\delta$ is too large, the relaxation bound is used also for
large boxes, for which it gives poor results (since the error grows
with the square of the box size, as stated in
Theorem~\ref{prop:relaxation_error}). Hence, the computation of
the relaxation bound for larger boxes is ineffectual and has the only
effect of increasing the
computation time. Note that the number of nodes is not increased,
since, in any case, the cheap bound is computed for all boxes and the
overall lower bound for each box is the minimum between the cheap
bound and the relaxation bound. Note that these experiments clearly show that the relaxation bound is essential for the efficiency of the algorithm.
\begin{figure}[h]
    \centering
    \includegraphics[width=0.95\columnwidth]{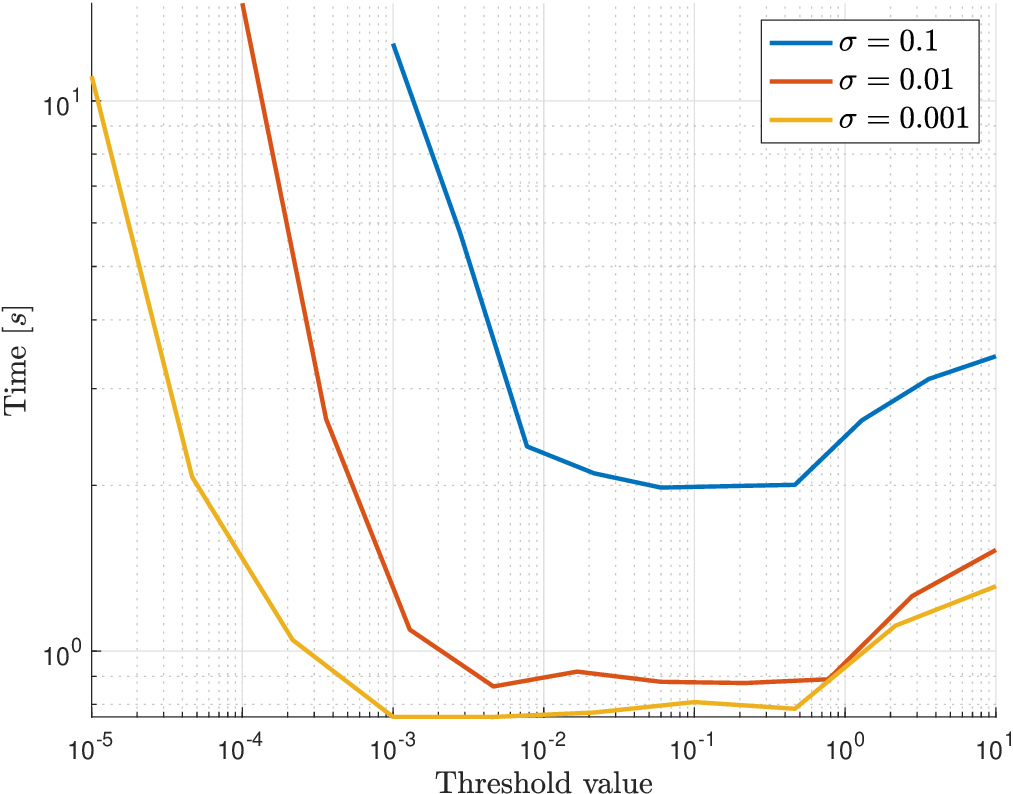}
    \caption{Computation times for randomly generated problems as a
      function of threshold $\delta$.}
    \label{fig:test1_timetr}
  \end{figure}  
\begin{figure}[h]
    \centering
    \includegraphics[width=0.95\columnwidth]{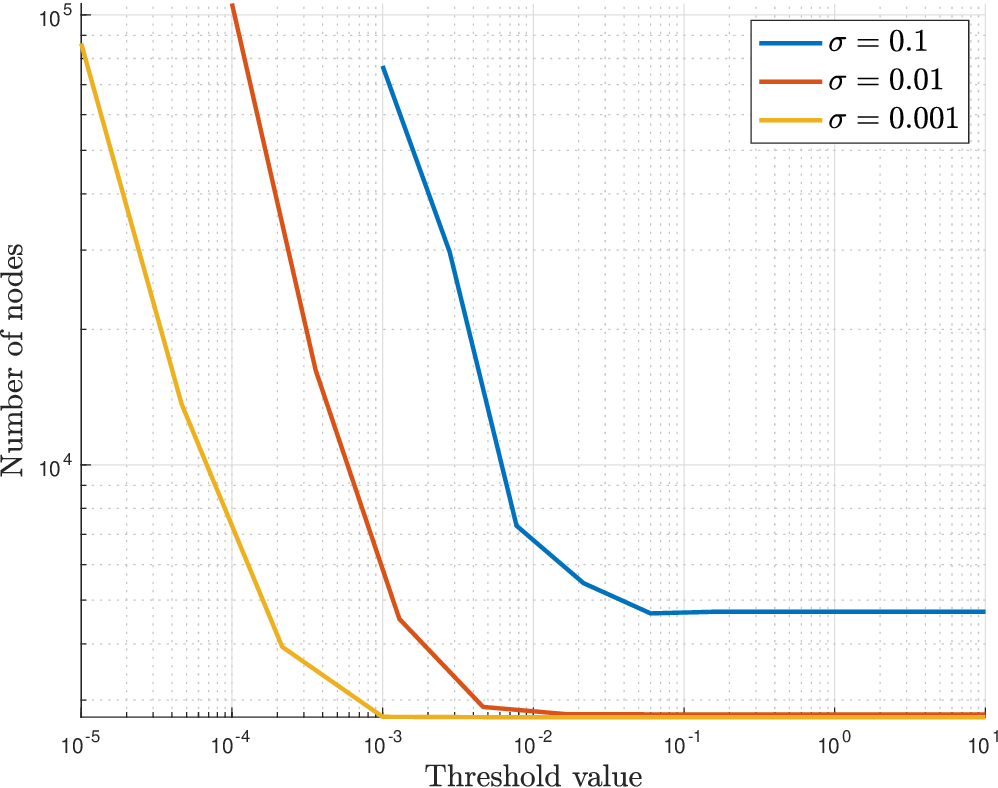}
    \caption{Number of explored nodes for randomly generated problems as a
      function of threshold $\delta$.}
    \label{fig:test1_nodestr}
\end{figure}
We also performed experiments to assess the impact of the proposed
queue management algorithm, presented in Section~\ref{sec:algorithm}
and Algorithm~\ref{alg_bnb_imp}, on computational efficiency.
To this end,  we solved random instances of
Problem~\ref{prob:glores}, obtained in the same way as in the previous
tests, for different values of standard
deviation $\sigma$, without the use of the queue, computing all items
of $L_{B,C}$ from scratch at each iteration, as done in the
initialization algorithm.
Figures~\ref{fig:test1_time_noc} and~\ref{fig:test1_nodes_noc} illustrate the results of these experiments. 
Note that the execution time of the algorithm with the queue algorithm
is significantly lower than the version without the efficient queue
management. Note also that we did not perform tests for a number of
point $n$ larger than $50$ because the computation time would have
been too large.
\begin{figure}[h]
    \centering
    \includegraphics[width=0.95\columnwidth]{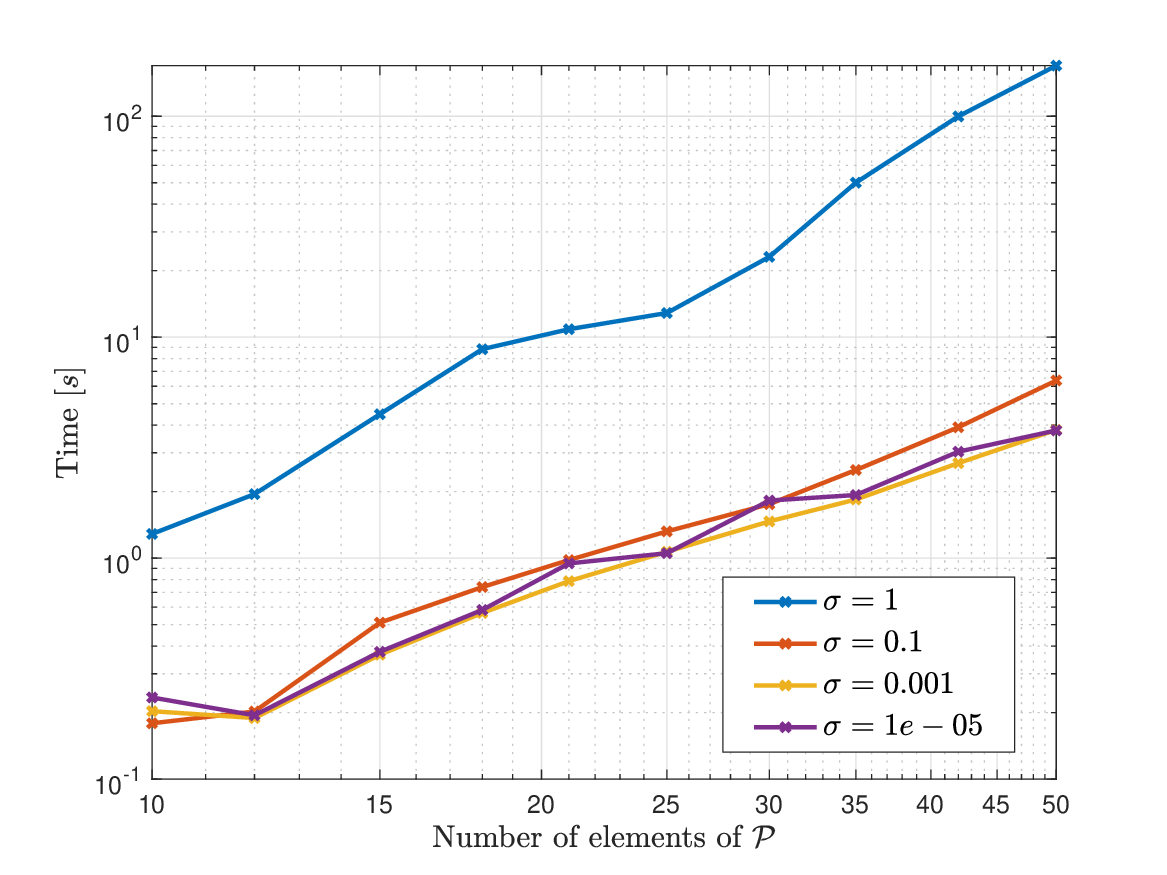}
    \caption{Computation times for randomly generated problems without
    queue algorithm.}
    \label{fig:test1_time_noc}
  \end{figure}  
\begin{figure}[h]
    \centering
    \includegraphics[width=0.95\columnwidth]{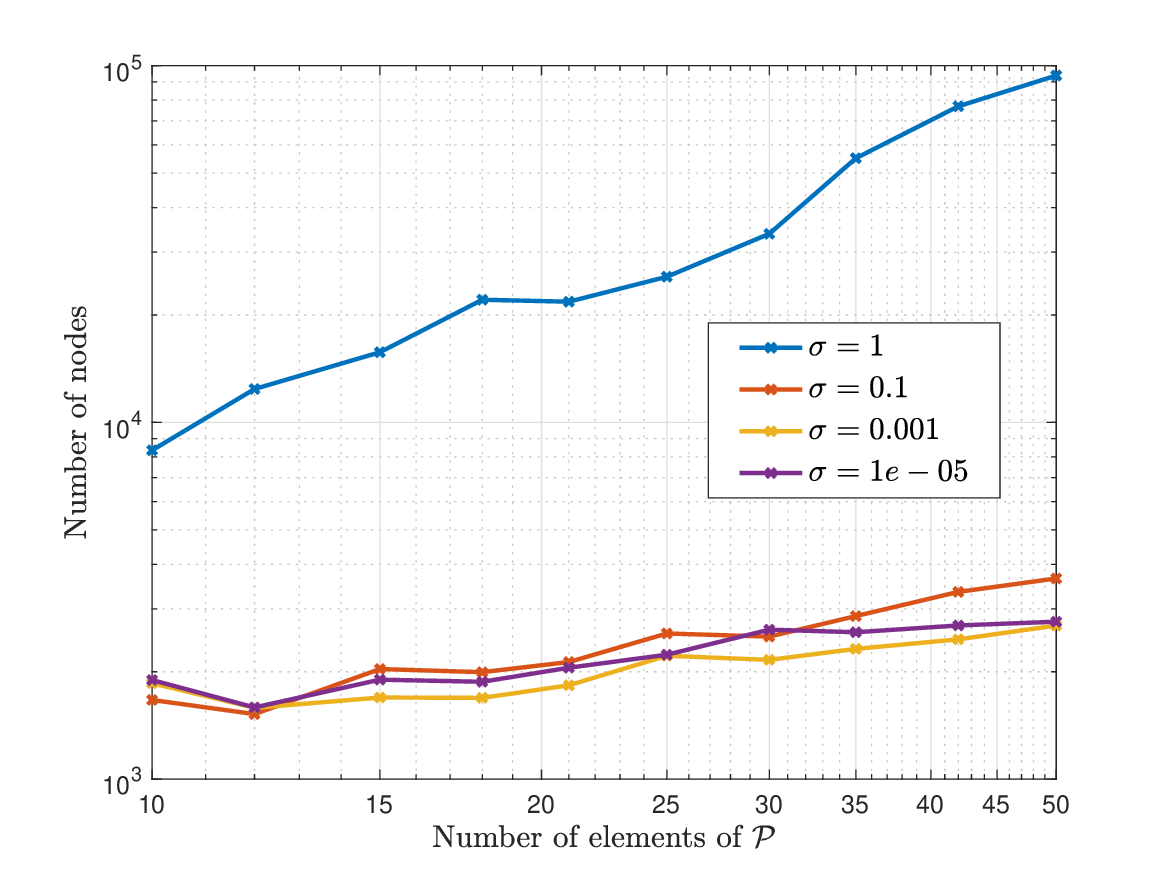}
    \caption{Number of explored nodes for randomly generated problems without
    queue algorithm.}
    \label{fig:test1_nodes_noc}
\end{figure}
\subsection{Laser scanner data}
We obtained some real world data by a lidar sensor Sick NAV350 mounted on
an industrial autonomous vehicle. We acquired different point sets, denoted by
$\Lc_1,\ldots,\Lc_S$, by placing the vehicle at various locations in a
warehouse. Each set contains $n=480$ points. It is related to a different vehicle pose
(defined by its position and orientation) and each point represents
the position of the first obstacle encountered by the lidar laser beam
along a given direction. Two such set of points are reported in Figure~\ref{fig:scan_aligned_e80_relaxation_on}.
We considered a set of $S=15$ poses and, for each couple of poses
$(i,j)$, we solved Problem~\ref{prob:glores} with $\Pc=\Lc_i$ and
$\Qc=\Lc_j$, computing the solution time $t_{(i,j)}$ and the optimal
value $f_{(i,j)}$.
\setlength\tabcolsep{1.5pt}
\input{test_2_times.tex}
\input{test_2_fun.tex}
\begin{figure}[h]
    \centering
    \includegraphics[width=0.95\columnwidth]{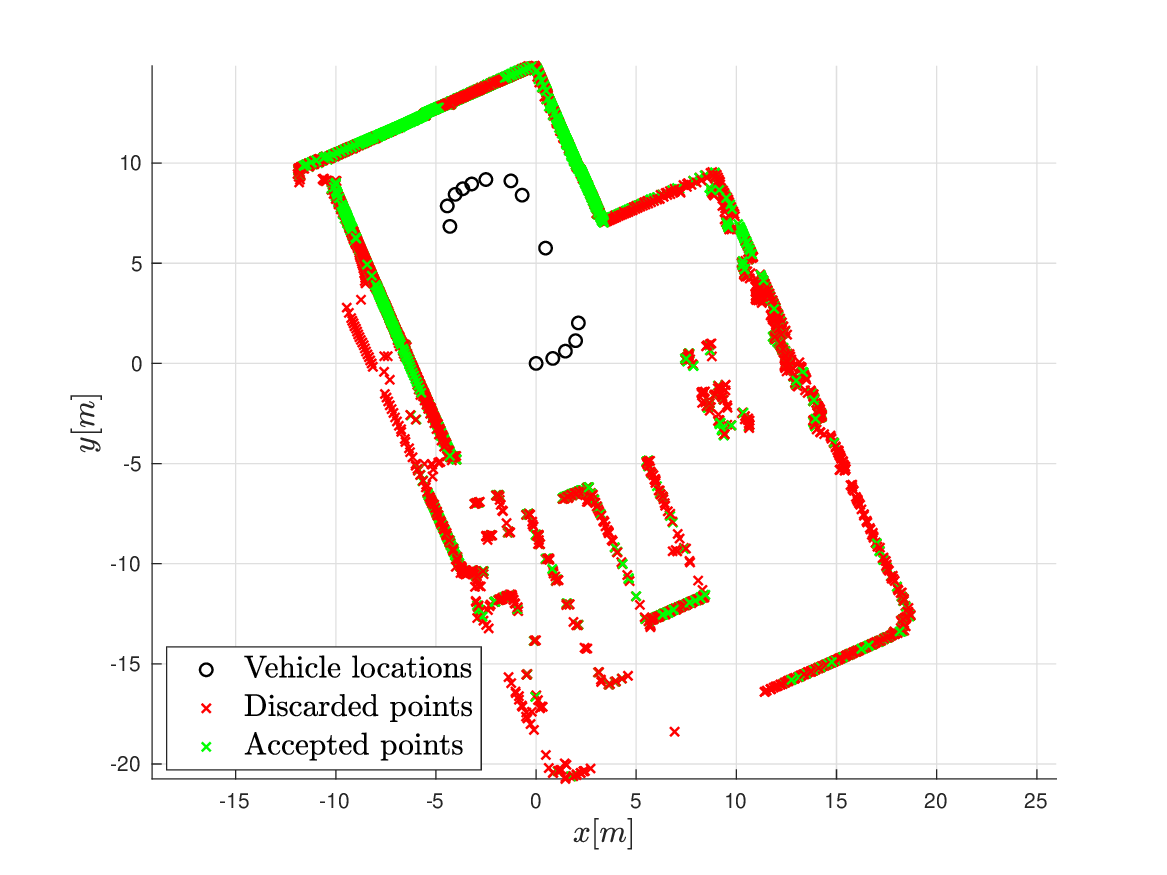}
    \caption{Reconstructed map from laser scanner data.}
    \label{fig:test2_map}
\end{figure}
In our tests, we set $\epsilon=0.0001$, $p=\lceil 0.8 n \rceil$ and
$\delta=0.1$.
We used the same hardware as in the randomly generated tests.
The results are reported in Tables~\ref{table:table_times} and~\ref{table:table_fun}. 
For the sake of representation, we computed the reconstructed vehicle
poses and sensed point with respect to the reference frame of the
first pose.  To this end, we defined an undirected complete graph with nodes $\{1,\ldots,S\}$ and assigned weight $f_{(i,j)}$ to arc $(i,j)$.
To define a common reference frame for all poses, we arbitrarily set the pose of $\Lc_1$ at
$(0,0)$ with $0$ orientation. Then, we computed the location and orientation
of the origin of every set $\Lc_i$, $i \in \{2,\ldots,S\}$, with the
following procedure. Let $P$ be the minimum distance path that
joins node $i$ to node $1$. We compute the coordinate transformation from
the origin of $\Lc_i$ to the origin of $\Lc_1$ by
composing the transformations associated to the arcs of
$P$. The overall error is minimized since $P$ is the minimum distance
path.
The results are reported in Figure~\ref{fig:test2_map}, where the
green crosses represent the transformed points that have been included
in the sum in~\eqref{eq:obj_fun}, the red crosses represent the points
that have not been included in the sum in~\eqref{eq:obj_fun} (the
outliers) and the circles represents the estimated positions of the
autonomous vehicle.
\begin{figure}[th]
    \centering
    \includegraphics[width=0.95\columnwidth]{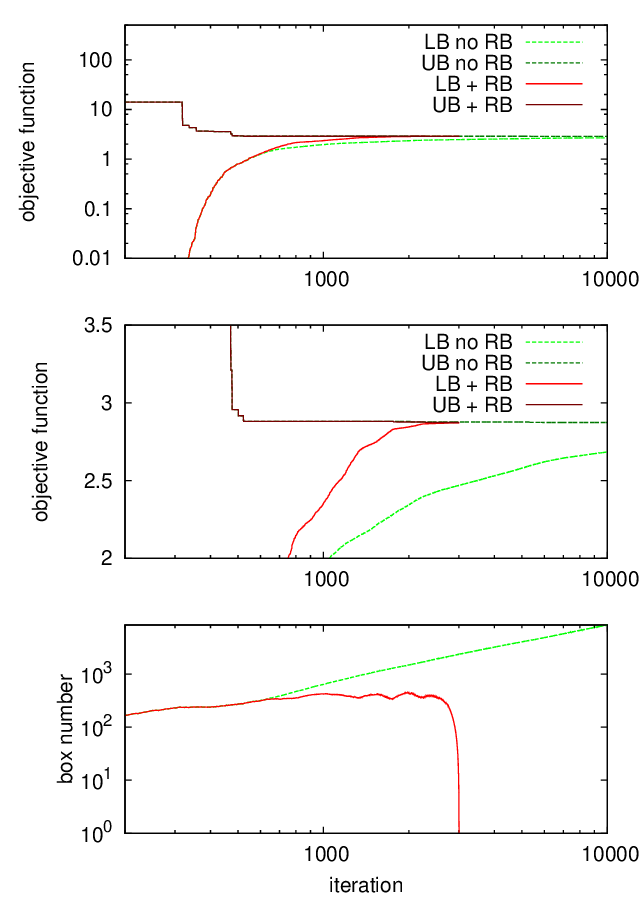}
    \caption{Comparison between BnB registration with or without the relaxation bound (RB).
      First row: the lower (LB) and upper (UB) bounds of objective error function estimated with (red) or without (green) the RB.
      Second row: an enlargement of previous plot focused on objective function global minimum. 
      Third row: the number of nodes used by the algorithms with or without the relaxation bound. 
    }
    \label{fig:stats_bounds_scan_e80}
\end{figure}
We also applied Algorithm~\ref{alg_bnb_imp} to the point sets reported
in Figure~\ref{fig:scan_aligned_e80_relaxation_on} with and without
the use of the relaxation bound. Namely, in a first test we set
$\delta=0.8$ and $\epsilon=0.001$, while, in a second test, we changed
the value of $\delta$ to $0$, disabling the use of the relaxation bound.
In Figure~\ref{fig:stats_bounds_scan_e80}, the top and the middle
graph (which is a magnification of the top one) report the results of these
two tests. In particular, the line denoted by ``RB'' represent the
first test, with $\delta=0.8$ and the line denoted by ``no-RB''
represents the second one, with the relaxation bound disabled.
Note that, with $\delta=0.8$, the lower and upper bounds converge to
required relative tolerance $\epsilon = 0.001$ after iteration $2123$
whereas, with $\delta=0$, $10000$ iterations are not sufficient. 
The bottom graph of Figure~\ref{fig:stats_bounds_scan_e80} shows the
number of boxes in the queue for the two cases. 
Note that the number of boxes obtained in the case $\delta=0.8$ is
much smaller than in the case $\delta=0$.
\section{Conclusions}
\label{sec:conclusion}

In this paper, we presented a BnB algorithm for the registration of planar point sets. 
The algorithm estimates the global minimizer of the objective function given by the sum of distances between each point of the source set to its closest point on the destination set. 
The main contribution w.r.t. the state-of-the-art global registration lies in the adoption 
of a novel lower bound, the \emph{relaxation bound}, which is based on linearized distance and speeds up the convergence of the registration through local linear approximation function exploiting mutual evaluation on the whole point set. 
Such lower bound is very effective on small boxes, while on larger boxes we adopted  the \emph{cheap bound}, which decomposes the estimation of objective function into the sum of the minimum distances between corresponding points, which considerably improves the performance. 
This is similar to what already done in~\cite{YangLiCampbellJia16} but 
the efficient update and dropping of the candidate closest points is driven by formally guaranteed policies. 
The correctness and speed of convergence of the proposed bounds as well as the global optimality of the solution are formally proven. 
Simulation and real datasets acquired with range finders have shown the accuracy and efficiency of the proposed method. 
In future works, we expect to address the registration problem in 3D space using the same approach to lower bound. 
The decomposition of error function used in cheap bound and the relaxation bound are general and can be applied to more geometrically complex problems. 

\bibliographystyle{abbrv}
\bibliography{biblio}

\end{document}

%% file: test_2_times.tex
\begin{table}
\centering
\begin{tabular}{|c|c|c|c|c|c|c|c|c|c|c|c|c|c|}
\hline
 & 1 & 2 & 3 & 4 & 5 & 6 & 7 & 8 & 9 & 10 & 11 & 12 & 13 \\
\hline
2 & 40.8 &  &  &  &  &  &  &  &  &  &  &  &  \\
\hline
3 & 36.6 & 43.0 &  &  &  &  &  &  &  &  &  &  &  \\
\hline
4 & 42.7 & 41.9 & 38.5 &  &  &  &  &  &  &  &  &  &  \\
\hline
5 & 41.9 & 40.1 & 46.9 & 49.0 &  &  &  &  &  &  &  &  &  \\
\hline
6 & 43.7 & 42.1 & 50.8 & 41.6 & 39.2 &  &  &  &  &  &  &  &  \\
\hline
7 & 36.9 & 37.8 & 37.4 & 39.3 & 37.4 & 31.3 &  &  &  &  &  &  &  \\
\hline
8 & 63.3 & 47.7 & 51.9 & 58.5 & 62.8 & 35.7 & 44.3 &  &  &  &  &  &  \\
\hline
9 & 39.0 & 39.2 & 33.5 & 40.7 & 38.7 & 28.7 & 30.4 & 28.6 &  &  &  &  &  \\
\hline
10 & 31.3 & 34.1 & 31.3 & 35.1 & 31.4 & 25.6 & 25.7 & 24.4 & 24.1 &  &  &  &  \\
\hline
11 & 30.3 & 43.4 & 28.5 & 33.8 & 27.5 & 23.2 & 23.0 & 24.0 & 23.1 & 23.5 &  &  &  \\
\hline
12 & 27.6 & 27.2 & 28.1 & 29.1 & 26.9 & 24.0 & 26.3 & 23.4 & 24.7 & 24.2 & 23.9 &  &  \\
\hline
13 & 41.6 & 33.0 & 36.8 & 38.1 & 36.1 & 31.5 & 30.6 & 31.0 & 32.0 & 29.6 & 30.0 & 28.8 &  \\
\hline
14 & 29.2 & 32.1 & 30.6 & 28.4 & 27.8 & 24.9 & 28.5 & 27.9 & 27.5 & 28.9 & 27.9 & 27.0 & 29.3 \\
\hline
\end{tabular}
\caption{Computation times [s] for each pair of poses}
\label{table:table_times}
\end{table}

%% file: test_2_fun.tex
\begin{table}
\centering
\begin{tabular}{|c|c|c|c|c|c|c|c|c|c|c|c|c|c|}
\hline
 & 1 & 2 & 3 & 4 & 5 & 6 & 7 & 8 & 9 & 10 & 11 & 12 & 13 \\
\hline
2 & 0.3 &  &  &  &  &  &  &  &  &  &  &  &  \\
\hline
3 & 0.4 & 0.3 &  &  &  &  &  &  &  &  &  &  &  \\
\hline
4 & 0.4 & 0.3 & 0.4 &  &  &  &  &  &  &  &  &  &  \\
\hline
5 & 0.5 & 0.4 & 0.5 & 0.4 &  &  &  &  &  &  &  &  &  \\
\hline
6 & 0.9 & 0.8 & 0.9 & 0.9 & 0.8 &  &  &  &  &  &  &  &  \\
\hline
7 & 1.5 & 1.4 & 1.3 & 1.4 & 2.3 & 0.6 &  &  &  &  &  &  &  \\
\hline
8 & 7.3 & 4.5 & 3.5 & 4.1 & 8.6 & 1.2 & 0.4 &  &  &  &  &  &  \\
\hline
9 & 3.1 & 2.1 & 1.7 & 2.1 & 4.9 & 0.8 & 0.2 & 0.2 &  &  &  &  &  \\
\hline
10 & 1.5 & 1.4 & 1.4 & 1.5 & 2.8 & 0.7 & 0.2 & 0.3 & 0.1 &  &  &  &  \\
\hline
11 & 1.3 & 1.2 & 1.3 & 1.5 & 2.8 & 0.7 & 0.2 & 0.3 & 0.1 & 0.1 &  &  &  \\
\hline
12 & 1.1 & 1.1 & 1.2 & 1.3 & 2.1 & 0.6 & 0.2 & 0.3 & 0.2 & 0.1 & 0.1 &  &  \\
\hline
13 & 1.3 & 1.3 & 1.4 & 1.5 & 2.6 & 0.8 & 0.2 & 0.4 & 0.2 & 0.2 & 0.1 & 0.1 &  \\
\hline
14 & 1.3 & 1.6 & 1.8 & 1.5 & 2.3 & 0.5 & 0.2 & 0.4 & 0.2 & 0.2 & 0.2 & 0.2 & 0.2 \\
\hline
\end{tabular}
\caption{Minimum for each pair of poses}
\label{table:table_fun}
\end{table}